\documentclass[10pt,journal,cspaper,compsoc]{IEEEtran}
\usepackage[rgb]{xcolor}

\ifCLASSOPTIONcompsoc
\else
\fi
\ifCLASSINFOpdf
\else
\fi
\usepackage[cmex10]{amsmath}
\usepackage{amssymb}
\usepackage{amsthm}

\newtheorem{theorem}{Theorem}
\newtheorem{lemma}{Lemma}
\newtheorem{corollary}{Corollary}
\newtheorem{definition}{Definition}

\newtheorem{innercustomthm}{Theorem}
\newenvironment{theoremFixed}[1]
  {\renewcommand\theinnercustomthm{#1}\innercustomthm}
  {\endinnercustomthm}

\newtheorem{innercustomlem}{Lemma}
\newenvironment{lemmaFixed}[1]
  {\renewcommand\theinnercustomlem{#1}\innercustomlem}
  {\endinnercustomlem}

\usepackage{color}
\definecolor{darkblue}{rgb}{0,0.1,0.5}
\definecolor{black}{rgb}{0,0,0}
\usepackage{epsfig}

\usepackage{array}
\usepackage{rotating}
\usepackage{multirow}

\usepackage{url}
\usepackage[%
pdfpagelabels,hypertexnames=true,
plainpages=false,
            naturalnames=false,
            colorlinks,
            linkcolor=black,
            anchorcolor=black,
            citecolor=black,
            urlcolor=black
]%
{hyperref}
\usepackage[all]{hypcap}

\usepackage[numbers]{natbib}

\usepackage{fixmath}
\renewcommand{\vec}[1]{\mathbold{#1}}

\newcommand{\argmin}{\operatornamewithlimits{argmin}}
\newcommand{\argmax}{\operatornamewithlimits{argmax}}
\newcommand{\Argmin}{\operatornamewithlimits{Argmin}}

\renewcommand{\paragraph}[1]{\noindent\textbf{#1}}
\newcommand{\st}{\mathop{\mathrm{s.t.}}}

\usepackage[noend]{algorithmic}

\begin{document}
%
\title{Submodular relaxation for inference in Markov random fields}
%
%
%
%


\author{Anton~Osokin \qquad\qquad 
        Dmitry~Vetrov
\IEEEcompsocitemizethanks{\IEEEcompsocthanksitem
A. Osokin is with SIERRA team, INRIA and \'{E}cole Normale Sup\'{e}rieure, Paris, France.
E-mail: anton.osokin@inria.fr
\IEEEcompsocthanksitem D. Vetrov is with Faculty of computer science, Higher School of Economics, Moscow, Russia and Lomonosov Moscow State University, Moscow, Russia.
E-mail: vetrovd@yandex.ru}
\thanks{}}

\IEEEcompsoctitleabstractindextext{%
\begin{abstract}
In this paper we address the problem of finding the most
probable state of a discrete Markov random field (MRF), also known as the MRF energy minimization problem.
The task is known to be NP-hard in general and its practical importance motivates numerous approximate algorithms.
We propose a submodular relaxation approach (SMR) based on a Lagrangian relaxation of the initial problem. Unlike the dual decomposition approach of~\citet{Komodakis10ddtrw} SMR does not decompose the graph structure of the initial problem but constructs a submodular energy that is minimized within the Lagrangian relaxation. Our approach is applicable to both pairwise and high-order MRFs and allows to take into account global potentials of certain types.
We study theoretical properties of the proposed approach and evaluate it experimentally.
\end{abstract}

}

\maketitle

\IEEEdisplaynotcompsoctitleabstractindextext

%
\IEEEpeerreviewmaketitle

\section{Introduction \label{sec::intro}}
The problem of inference in a Markov random field (MRF) arises in many applied domains, e.g. in machine learning, computer vision, natural language processing, etc.
In this paper we focus on one important type of inference: maximum a posteriori (MAP) inference, often referred to as MRF energy minimization.
Inference of this type is a combinatorial optimization problem, i.e. an optimization problem with the finite domain.

The most studied case of MRF energy minimization is the situation when the energy can be represented as a sum of terms (potentials) that depend on only one or two variables each (unary and pairwise potentials).
In this setting the energy is said to be defined by a graph where the nodes correspond to the variables and the edges to the pairwise potentials.
Minimization of energies defined on graphs in known to be NP-hard in general~\cite{Boykov01} but can be done exactly in polynomial time in a number of special cases, e.g. if the graph defining the energy is acyclic~\cite{Pearl88} or if the energy is submodular in standard~\cite{Kolmogorov04} or multi-label sense~\cite{Darbon09}.


One way to go beyond pairwise potentials is to add higher-order summands to the energy.
For example, \citet{kohli2009robust} and \citet{Ladicky09ale} use  high-order potentials based on superpixels (image regions) for semantic image segmentation; \citet{Delong12ijcv} use label cost potentials for geometric model fitting tasks.

To be tractable, high-order potentials need to have a compact representation.
The recently proposed sparse high-order potentials define specific values for a short list of preferred configurations and the default value for all the others~\cite{Komodakis09highorder}, \cite{Rother09highorder}.

In this paper we develop the submodular relaxation framework (SMR) for approximate minimization of energies with pairwise and sparse high-order potentials.
SMR is based on a Lagrangian relaxation of consistency constraints on labelings, i.e. constraints that make each node to have exactly one label assigned.
The SMR Lagrangian can be viewed at as a pseudo-Boolean function of binary indicator variables and is often submodular, thus can be efficiently minimized with a max-flow/min-cut algorithm (graph cut) at each step of the Lagrangian relaxation process.

In this paper we explore the SMR framework theoretically and provide explicit expressions for the obtained lower bounds.
We experimentally compare our approach against several baselines (decomposition-based approaches~\cite{Komodakis10ddtrw}, \cite{Kolmogorov06trws}, \cite{Komodakis09highorder}) and show its applicability on a number of real-world tasks.

The rest of the paper is organized as follows. In sec.~\ref{sec::relatedWorks} we review the related works. In sec.~\ref{sec::notation} we introduce our notation and discuss some well-known results. In sec.~\ref{sec::smr} we present the SMR framework. In sec.~\ref{sec::theorOptimization} we construct several algorithms within the SMR framework. In sec.~\ref{sec::theory} we present the theoretical analysis of our approach and in sec.~\ref{sec::experiments} its experimental evaluation. We conclude in sec.~\ref{sec::conclusion}.

\section{Related works \label{sec::relatedWorks}}
We mention works related to ours in two ways. First, we discuss methods that analogously to SMR rely on multiple calls of graph cuts.
Second, we mention some decomposition methods based on Lagrangian relaxation.

\paragraph{Multiple graph cuts.} The SMR framework uses the one-versus-all encoding for the variables, i.e. the indicators that particular labels are assigned to particular variables, and runs the graph cut at the inner step. In this sense our method is similar to the $\alpha$-expansion algorithm~\cite{Boykov01} (and its generalizations \cite{Kohli08potts}, \cite{kohli2009robust}) and the methods for minimizing multi-label submodular functions~\cite{Ishikawa03}, \cite{Darbon09}. Nevertheless SMR is quite different from these approaches.  The $\alpha$-expansion algorithm is a move-making method where each move decreases the energy. Methods of~\cite{Ishikawa03} and~\cite{Darbon09} are exact methods applicable only when a total order on the set of labels is defined and the pairwise potentials are convex w.r.t. it (or more generally submodular in multi-label sense). Note that the frequently used Potts model is not submodular when the number of possible labels is larger than~2, so the methods of~\cite{Ishikawa03} and \cite{Darbon09} are not applicable to it.

\citet{Jegelka11,Kohli13} studied the high-order potentials that group lower-order potentials together (cooperative cuts).
Similarly to SMR cooperative cut methods rely on calling graph cuts multiple times.
As opposed to SMR these methods construct and minimize upper (not lower) bounds on the initial energy.

The QPBO method~\cite{Boros02} is a popular way of constructing the lower bound on the global minimum using graph cuts.
When the energy depends on binary variables the SMR approach obtains exactly the same lower bound as QPBO.


\paragraph{Decomposition methods}
split the original hard problem into easier solvable problems and try to make them reach an agreement on their solutions. One advantage of these methods is the ability not only to estimate the solution but also to provide a lower bound on the global minimum.
Obtaining a lower bound is important because it gives an idea on how much better the solution could possibly be.
If the lower bound is tight (the obtained energy equals the obtained lower bound) the method provides a \emph{certificate of optimality}, i.e. ensures that the obtained solution is optimal.
The typical drawback of such methods is a higher computational cost in comparison with e.g. the $\alpha$-expansion algorithm~\cite{Kappes13comparison}.

Usually a decomposition method can be described  by the two features: a way to enforce agreement of the subproblems and type of subproblems.
We are aware of two ways to enforce agreement: message passing and Lagrangian relaxation. Message passing was used in this context by \citet{Wainwright05trw}, \citet{Kolmogorov06trws} in the tree-reweighted message passing (TRW) algorithms.
The Lagrangian relaxation was applied to energy minimization by \citet{Komodakis10ddtrw} and became very popular because of its flexibility.
The SMR approach uses the Lagrangian relaxation path.

The majority of methods based on the Lagrangian relaxation decompose the graph defining the problem into subgraphs of simple structure.
The usual choice is to use acyclic graphs (trees) or, more generally, graphs of low treewidth \cite{Komodakis10ddtrw}, \cite{Batra10outer}.
\citet{Komodakis10ddtrw} pointed out the alternative option of using submodular subproblems but only in case of problems with binary variables and together with the decomposition of the graph.
\citet{Strandmark10} used decomposition into submodular subproblems to construct a distributed max-flow/min-cut algorithm for large-scale binary submodular problems.
\citet{Komodakis09highorder} generalized the method of~\cite{Komodakis10ddtrw} to work with the energies with sparse high-order potentials. In this method groups of high-order potentials were put into separate subproblems.

Several times the Lagrangian relaxation approach was applied to the task of energy minimization without decomposing the graph of the problem.
\citet{Yarkony11planar} created several copies of each node and enforced copies to take equal values.
\citet{Yarkony11tightening} introduced planar binary subproblems in one-versus-all way (similarly to SMR) and used them in the scheme of tightening the relaxation.

The SMR approach is similar to the methods of \cite{Batra10outer}, \cite{Komodakis09highorder}, \cite{Komodakis10ddtrw}, \cite{Strandmark10}, \cite{Yarkony11planar}, \cite{Yarkony11tightening} in a sense that all of them use the Lagrangian relaxation to reach an agreement.
The main advantage of the SMR is that the number of dual variables does not depend on the number of labels and the number of high-order potentials but is fixed to the number of nodes. This structure results in the speed improvements over the competitors.

\section{Notation and preliminaries \label{sec::notation}}
\paragraph{Notation.}
Let~$\mathcal{V}$ be a finite set of variables. Each variable $i \in \mathcal{V}$ takes values $x_i \in \mathcal{P}$ where $\mathcal{P}$ is a finite set of labels. For any subset of variables $C \subseteq \mathcal{V}$ the Cartesian product $\mathcal{P}^{|C|}$ is the \emph{joint domain} of variables~$C$. A \emph{joint state} (or a \emph{labeling}) of variables $C$ is a mapping $\vec{x}_C: C \to \mathcal{P}$ and $\vec{x}_C(i)$, $i \in C$  is the image of $i$ under  $\vec{x}_C$. We use $\mathcal{P}^C$ to denote a set of all possible joint states of variables in~$C$.\!\footnote{
Note that $\mathcal{P}^C$ is highly related to $\mathcal{P}^{|C|}$. The ``mapping terminology'' is introduced to allow original variables $i \in C \subseteq \mathcal{V}$  to be used as arguments of labelings: $\vec{x}_{C}(i)$.
}
A joint state of all variables in~$\mathcal{V}$ is denoted by~$\vec{x}$ (instead of $\vec{x}_{\mathcal{V}}$). If an expression contains both $\vec{x}_C \in \mathcal{P}^C $ and $\vec{x}_A \in \mathcal{P}^A$, $A, C \subseteq \mathcal{V}$ we assume that the mappings are consistent, i.e. $\forall i \in A \cap C : \vec{x}_{C}(i) = \vec{x}_{A}(i)$.
%

Consider hypergraph $(\mathcal{V}, \mathcal{C})$ where $\mathcal{C} \subseteq 2^\mathcal{V}$ is a set of hyperedges.
A \emph{potential} of hyperedge $C$ is a mapping
$
\theta_C: \mathcal{P}^C \to \mathbb{R}
$ and a number $\theta_C( \vec{x}_C ) \in \mathbb{R}$ is the value of the potential on labeling $\vec{x}_C$.
The \emph{order} of potential $\theta_C$ refers to the cardinality of set~$C$ and is denoted by $|C|$.
Potentials of order 1 and 2  are called \emph{unary} and \emph{pairwise} correspondingly. Potentials of order higher than 2 are called \emph{high-order}.

An \emph{energy} defined on hypergraph $(\mathcal{V}, \mathcal{C})$ (\emph{factorizable} according to hypergraph $(\mathcal{V}, \mathcal{C})$) is a mapping $E: \mathcal{P}^\mathcal{V} \to \mathbb{R}$ that can be represented as a sum of potentials of hyperedges~$C \in \mathcal{C}$:
\begin{equation}
\label{eq::notation::energyMultilabel}
E(\vec{x}) = \sum_{C \in \mathcal{C}} \theta_C( \vec{x}_C).
\end{equation}
It is convenient to use notational shortcuts related to potentials and variables: $\theta_{i_1\dots i_k} (x_{i_1},\dots, x_{i_k}) = \theta_{i_1\dots i_k, x_{i_1},\dots, x_{i_k}} = \theta_{C, \vec{x}_C} =\theta_C( \vec{x}_C)$, $\vec{x}_C(i) = x_{C, i}$  where $C = \{ i_1,\dots, i_k \}$ and $i\in C$.

When an energy is factorizable to unary and pairwise potentials only we call it \emph{pairwise}. We write pairwise energies as follows:
\begin{equation}
\label{eq::notation::energyMultilabelPairwise}
E(\vec{x}) = \sum_{i \in \mathcal{V}} \theta_i( x_i) + \sum_{\{i, j\} \in \mathcal{E}} \theta_{ij}( x_i, x_j).
\end{equation}
Here $\mathcal{E}$ is a set of \emph{edges}, i.e. hyperedges of order 2. Under this notation set $\mathcal{C}$ contains hyperedges of order 1 and 2:
$
\mathcal{C} = \{ \{i\} \mid i \in \mathcal{V} \} \cup \{ \{i, j\} \mid \{i, j\} \in \mathcal{E} \}
$.

Sometimes it will be convenient to use indicators of each variable taking each value: $y_{ip} = [x_i\!\!=\!p]$, $i \in \mathcal{V}$, $p \in \mathcal{P}$. Indicators allow us to rewrite the energy~\eqref{eq::notation::energyMultilabel} as
\begin{equation}
\label{eq::notation::energyIndicators}
E_I( \vec{y} ) = \sum_{C \in \mathcal{C}} \sum_{\vec{d} \in \mathcal{P}^C} \theta_{C, \vec{d}} \prod_{i \in C} y_{i d_{i}}.
\end{equation}

Minimizing \emph{multilabel} energy~\eqref{eq::notation::energyMultilabel} over $\mathcal{P}^{\cal V}$ is equivalent to minimizing \emph{binary} energy~\eqref{eq::notation::energyIndicators} over the Boolean cube $\{ 0, 1 \}^{\mathcal{V} \times \mathcal{P}}$ under so-called \emph{consistency constraints}:
\begin{equation}
\label{eq::notation::consistencyConstraints}
\sum_{p \in \mathcal{P}} y_{ip} = 1, \quad \forall i \in \mathcal{V}.
\end{equation}
Throughout this paper, we use equation numbers to indicate a constraint given by the corresponding equation, for example $\vec{y} \in \eqref{eq::notation::consistencyConstraints}$.

\paragraph{The Lagrangian relaxation approach.}
A popular approximate approach to minimize energy~\eqref{eq::notation::energyMultilabel} is to enlarge the variables' domain, i.e. perform the \emph{relaxation}, and to construct the dual problem to the relaxed one as it might be easier to solve. Formally, we get
\begin{equation}
\label{eq::gaps}
\max_{\vec{\lambda}} D(\vec{\lambda})  \leq \min_{\vec{y} \in \mathcal{Y}} E_{lb}(\vec{y}) \leq \min_{\vec{x} \in\mathcal{X}} E(\vec{x})
\end{equation}
where $\mathcal{X}$ and $\mathcal{Y}$ are discrete and continuous sets, respectively, where $\mathcal{X} \subset \mathcal{Y}$. $E_{lb}(\vec{y})$ is a lower bound function defined on set~$\mathcal{Y}$ and $D(\vec{\lambda})$ a dual function.
Both inequalities in~\eqref{eq::gaps} can be strict.
If the left one is strict we say that there is a non-zero \emph{duality gap}, if the right one is strict~-- there is a non-zero \emph{integrality gap}.
Note that if the middle problem in~\eqref{eq::gaps} is convex and one of the regularity conditions is satisfied (e.g. it is a linear program) there is no duality gap.
The goal of all relaxation methods is to make the joint gap in~\eqref{eq::gaps} as small as possible.
In this paper we will construct the dual directly from the discrete primal and will need the continuous primal only for the theoretical analysis.


\paragraph{The linear relaxation of a pairwise energy}
consists in converting the energy minimization problem to an equivalent \emph{integer linear program (ILP)} and in relaxing it to a \emph{linear program (LP)}~\cite{Schlesinger76}, \cite{Wainwright05trw}.

\citet{Wainwright05trw} have shown that minimizing energy~\eqref{eq::notation::energyMultilabelPairwise} over the discrete domain is equivalent to minimizing the linear function
\begin{equation}
\label{eq::notation::energyLinear}
E_{L}(\vec{y}_L) = \sum_{i \in \mathcal{V}} \sum_{p \in \mathcal{P}} \theta_{ip} y_{ip} + \sum_{(i, j) \in \mathcal{E}} \sum_{p,q \in \mathcal{P}} \theta_{ij,pq} y_{ij,pq}
\end{equation}
w.r.t. unary~$y_{ip} \geq 0$ and pairwise~$y_{ij,pq} \geq 0$ variables over the so-called \emph{marginal polytope}: a linear polytope with facets defined by all marginalization constraints. The marginal polytope has an exponential number of facets, thus minimizing~\eqref{eq::notation::energyLinear} over it is intractable.
A polytope defined using only the first-order marginalization constraints is called a \emph{local polytope}:
\begin{align}
\label{eq::notation::localGConsistensy}
&\sum_{p \in \mathcal{P}} y_{ip} = 1, \quad \forall i\!\in\!\mathcal{V}, \\
\label{eq::notation::localGMarginalization}
&\sum_{p \in \mathcal{P}} y_{ij,pq}\!=\!y_{jq},\:\sum_{q \in \mathcal{P}} y_{ij,pq}\!=\!y_{ip}, \: \forall \{i,j\}\!\in\!\mathcal{E},\:\forall p,q\!\in\!\mathcal{P}.
\end{align}
We refer to constraints~\eqref{eq::notation::localGConsistensy} as \emph{consistency constraints} and to constraints~\eqref{eq::notation::localGMarginalization} as \emph{marginalization constraints}.

Minimizing~\eqref{eq::notation::energyLinear} over local polytope~\eqref{eq::notation::localGConsistensy}-\eqref{eq::notation::localGMarginalization} is often called a Schlesinger's LP relaxation~\cite{Schlesinger76} (see e.g.~\cite{Werner07} for a recent review). In this paper we refer to this relaxation as the \emph{standard LP relaxation}.

\section{Submodular Relaxation \label{sec::smr}}
In this section we present our approach. We start with pairwise MRFs (sec.~\ref{sec::pairwiseMRFs})~\cite{Osokin11cvpr}, generalize to high-order potentials (sec.~\ref{sec::highOrderMRFs}) and to linear global constraints (sec.~\ref{sec::globalTerms}).
A version of SMR for nonsubmodular Lagrangians is discussed in the supplementary material.
\subsection{Pairwise MRFs\label{sec::pairwiseMRFs}}
The indicator notation allows us to rewrite the minimization of pairwise energy~\eqref{eq::notation::energyMultilabelPairwise} as a constrained optimization problem:
\begin{align}
\label{eq::smd::problemPairwiseIndicator}
\min_{\vec{y}}&\quad \sum_{i \in \mathcal{V}} \sum_{p \in \mathcal{P}} \theta_{ip} y_{ip} + \sum_{ \{i, j\} \in \mathcal{E}} \sum_{p,q \in \mathcal{P}} \theta_{ij,pq} y_{ip} y_{jq},\\
\label{eq::smd::discretizationContraints}
\st
&\quad y_{ip} \in \{0, 1\}, \quad \forall i \in \mathcal{V}, p \in \mathcal{P},\\
\label{eq::smd::consistencyConstraints}
 &\quad \sum_{p \in \mathcal{P}} y_{ip} = 1, \quad \forall i \in \mathcal{V}.
\end{align}
Constraints~\eqref{eq::smd::consistencyConstraints} (i.e.~\eqref{eq::notation::consistencyConstraints})  guarantee each node~$i$ to have exactly one label assigned,  constraints~\eqref{eq::smd::discretizationContraints} fix variables~$y_{ip}$ to be binary.

The target function~\eqref{eq::smd::problemPairwiseIndicator} (denoted by $E_{I}(\vec{y})$) of binary variables~$\vec{y}$ (i.e.~\eqref{eq::smd::discretizationContraints} holds) is a pseudo-Boolean polynomial of degree at most~$2$. This function is easy to minimize (ignoring the consistency constraints~\eqref{eq::smd::consistencyConstraints}) if it is submodular, which is equivalent to having $\theta_{ij,pq} \leq 0$ for all edges~$\{i,j\} \in \mathcal{E}$ and all labels $p,q \in \mathcal{P}$. 
The so-called associative potentials are examples of potentials that satisfy these constraints.
A pairwise potential~$\theta_{ij}$ is~\emph{associative} if it rewards connected nodes for taking the same label, i.e.
\begin{equation}
\label{eq::smd::Potts}
\theta_{ij}(p, q) = -C_{ij,p} [p = q] =
\begin{cases}
-C_{ij,p}, & p = q, \\
0, & p \neq q.
\end{cases}
\end{equation}
where~$C_{ij,p} \geq 0$. Note that if constants~$C_{ij,p}$ are independent of label index~$p$ (i.e. $C_{ij,p}=C_{ij}$) then potentials~\eqref{eq::smd::Potts} are equivalent to the Potts potentials.

\citet{Osokin11cvpr} observe that when~$E_{I}(\vec{y})$ is submodular consistency constraints~\eqref{eq::smd::consistencyConstraints} only prevent problem~\eqref{eq::smd::problemPairwiseIndicator}-\eqref{eq::smd::consistencyConstraints} from being solvable, i.e. if we drop or perform the Lagrangian relaxation of them we obtain an easily computable lower bound. In other words we can look at the following Lagrangian:
\begin{equation}
\label{eq::smd::lagr}
L( \vec{y},\vec{\lambda})
=
E_I(\vec{y})+\sum_{i\in\mathcal{V}}\lambda_i\biggl(\sum_{p \in \mathcal{P}} y_{ip}-1\biggr).
\end{equation}
As a function of binary variables~$\vec{y}$ Lagrangian~$L( \vec{y},\vec{\lambda})$ differs from $E_I(\vec{y})$ only by the presence of unary potentials $\lambda_i y_{ip}$ and a constant, so Lagrangian~$L( \vec{y},\vec{\lambda})$ is submodular whenever $E_I(\vec{y})$ is submodular.

Applying the Lagrangian relaxation to the problem~\eqref{eq::smd::problemPairwiseIndicator}-\eqref{eq::smd::consistencyConstraints} we bound the solution from below
\begin{multline}
\label{eq::smd}
\min_{\vec{y} \in \eqref{eq::smd::discretizationContraints}, \eqref{eq::smd::consistencyConstraints}} E_I(\vec{y})
=
\min_{\vec{y} \in \eqref{eq::smd::discretizationContraints}} \max_{\vec{\lambda}} L(\vec{y}, \vec{\lambda})
\geq
\\
\max_{\vec{\lambda}}\min_{ \vec{y} \in \eqref{eq::smd::discretizationContraints}} L(\vec{y},\vec{\lambda})
=
\max_{\vec{\lambda}} D(\vec{\lambda})
\end{multline}
where $D(\vec{\lambda})$ is a \emph{Lagrangian dual function}:
\begin{equation}
\label{eq::smd::dualFunc}
D(\vec{\lambda}) = \min_{\vec{y}\in\eqref{eq::smd::discretizationContraints}} \biggl(E_I(\vec{y}) + \sum_{p \in \mathcal{P}}  \sum_{i\in\mathcal{V}}\lambda_i y_{ip}\biggr)-\sum_{i\in\mathcal{V}}\lambda_i.
\end{equation}
The dual function~$D(\vec{\lambda})$ is a minimum of a finite number of linear functions and thus is concave and piecewise-linear (non-smooth). A subgradient~$\vec{g} \in \mathbb{R}^{|\mathcal{V}|}$ of the dual~$D(\vec{\lambda})$ can be computed given the result of the inner pseudo-Boolean minimization:
\begin{equation}
\label{eq::dualSubgradient}
g_i = \sum_{p\in\mathcal{P}} y^*_{ip} - 1, \quad i \in \mathcal{V},
\end{equation}
where
$
\{y^*_{ip}\}_{i\in\mathcal{V}}^{p \in \mathcal{P}} = \argmin\limits_{\vec{y}\in\eqref{eq::smd::discretizationContraints}} \biggl(E_I(\vec{y}) + \sum\limits_{p \in \mathcal{P}}  \sum\limits_{i\in\mathcal{V}}\lambda_i y_{ip}\biggr).
$

The concavity of $D(\vec{\lambda})$ together with the ability to compute the subgradient makes it possible to tackle the problem of finding the best possible lower bound in family~\eqref{eq::smd} (i.e. the maximization of $D(\vec{\lambda})$) by convex optimization techniques.
The optimization methods  are discussed in sec.~\ref{sec::theorOptimization}.

Note, that if at some point~$\vec{\lambda}$ the subgradient given by eq.~\eqref{eq::dualSubgradient} equals zero it immediately follows that the inequality in~\eqref{eq::smd} becomes equality, i.e. there is no integrality gap and the certificate of optimality is provided. This effect happens because there exists a labeling~$\vec{y}^*$ delivering the minimum of~$L(\vec{y}, \vec{\lambda})$ w.r.t.~$\vec{y}$ and satisfying the consistency constraints~\eqref{eq::smd::consistencyConstraints}.

In a special case when all the pairwise potentials are associative (of form~\eqref{eq::smd::Potts}) minimization of the Lagrangian w.r.t.~$\vec{y}$  decomposes into $|\mathcal{P}|$ subproblems each containing only variables indexed by a particular label~$p$:
$$
\min_{\vec{y}\in\eqref{eq::smd::discretizationContraints}} L(\vec{y}, \vec{\lambda})
=
\sum_{p \in \mathcal{P}} \min_{\vec{y}_p\in\{0,1\}^{|\mathcal{V}|}} \Phi_p( \vec{y}_p, \vec{\lambda})
-
\sum_{i\in \mathcal{V}} \lambda_i
$$
where
$
\Phi_p( \vec{y}_p, \vec{\lambda}) =  \sum\limits_{i \in \mathcal{V}} (\theta_{ip} + \lambda_i) y_{ip}
-
\!\!\!\!\sum\limits_{\{i,j\}\in\mathcal{E}} \!\!C_{ij,p} y_{ip} y_{jp}.
$

In~this case the minimization tasks w.r.t $\vec{y}_p$ can be solved independently, e.g. in parallel.

\subsection{High-order MRFs \label{sec::highOrderMRFs}}
Consider a problem of minimizing the energy consisting of the high-order potentials~\eqref{eq::notation::energyMultilabel}.
The unconstrained minimization of energy~$E(\vec{x})$ over multi-label variables $\vec{x}$ is equivalent to the minimization of function~$E_I(\vec{y})$~\eqref{eq::notation::energyIndicators} over variables~$\vec{y}$ under the consistency constraints~\eqref{eq::smd::consistencyConstraints} and the discretization constraints~\eqref{eq::smd::discretizationContraints}.

For now let us assume that all the high-order potentials are \emph{pattern-based} and \emph{sparse}~\cite{Rother09highorder}, \cite{Komodakis09highorder}, i.e.
\begin{equation}
\label{eq::patternBased}
\theta_{C}(\vec{x}_C) =
\begin{cases}
\hat{\theta}_{C,\vec{d}}, \quad &\text{if $\vec{x}_C = \vec{d} \in \mathcal{D}_C$}, \\
0, \quad &\text{otherwise}.
\end{cases}
\end{equation}
Here set $\mathcal{D}_C$ contains some selected labelings of variables~$\vec{x}_C$ and all the values of the potential are non-positive, i.e.~$\hat{\theta}_{C,\vec{d}} \leq 0$, $\vec{d} \in \mathcal{D}_C$. In case of sparse potential the cardinality of set~$\mathcal{D}_C$ is much smaller compared to the number of all possible labelings of variables~$\vec{x}_C$, i.e $|\mathcal{D}_C| \ll |\mathcal{P}^{C}|$.

In this setting we can use the identity\\
\parbox{\columnwidth}{\centering
$
\left(-\prod_{\ell \in C} y_{\ell}\right)=\min_{z\in\{0,1\}}\left((|C|-1)z-\sum_{\ell \in C} y_{\ell} z\right)
$}\\
to perform a reduction\footnote{This transformation of binary high-order function~\eqref{eq::notation::energyIndicators} is in fact equivalent to a special case of  Type-II binary transformation of~\cite{Rother09highorder}, but in this form it was proposed much earlier (see e.g.~\cite{Ishikawa11} for a review).} of the high-order function~$E_I(\vec{y})$ to a pairwise function
in such a way that\\
\parbox{\columnwidth}{\centering
$\min_{\vec{y} \in \{0, 1\}^*} E_I(\vec{y}) = \min_{\vec{z} \in \{0, 1\}^*,\; \vec{y} \in \{0, 1\}^*} E_I^*(\vec{y}, \vec{z})$
}\\ where by $\{0, 1\}^*$ we denote the Boolean cubes of appropriate sizes and
\begin{multline}
\label{eq::smr::EYZsparseHighOrder}
E_I^*(\vec{y}, \vec{z}) =
-\sum_{C\in\mathcal{C}} \sum_{\vec{d}\in\mathcal{D}_C}
\hat{\theta}_{C,\vec{d}} \biggl((|C|-1) z_{C,\vec{d}}
\;- \\[-0.2cm]
\sum_{ \ell \in C} y_{\ell d_\ell} z_{C,\vec{d}} \biggr).
\end{multline}

Function~$E_I^*(\vec{y}, \vec{z})$ is pairwise and submodular w.r.t. binary variables $\vec{y}$ and $\vec{z}$ and thus in the absence of additional constraints can be efficiently minimized by graph cuts~\cite{Kolmogorov04}.
Adding and relaxing constraints~\eqref{eq::smd::consistencyConstraints} to the minimization of~\eqref{eq::notation::energyIndicators} gives us the following Lagrangian dual:
\begin{multline}
\label{eq::smr::dualHighOrder}
D(\vec{\lambda}) = \min_{\vec{y}, \vec{z}\in\{0,1\}^*}\!\!\!L( \vec{y}, \vec{z},\vec{\lambda}) = \\
\min_{\vec{y}, \vec{z} \in\{0,1\}^*} \biggl( E_I^*(\vec{y}, \vec{z}) + \sum_{i\in\mathcal{V}}\lambda_i\biggl(\sum_{p \in \mathcal{P}} y_{ip}-1\biggr)\biggr).
\end{multline}
For all $\vec{\lambda}$ Lagrangian~$L( \vec{y}, \vec{z},\vec{\lambda})$ is submodular w.r.t. $\vec{y}$ and $\vec{z}$  allowing us to efficiently evaluate $D$ at any point. Function~$D(\vec{\lambda})$ is a lower bound on the global minimum of energy~\eqref{eq::notation::energyMultilabel} and is concave but non-smooth as a function of~$\vec{\lambda}$. Analogously to~\eqref{eq::smd::dualFunc} the best possible lower bound can be found using convex optimization techniques.

\paragraph{Robust sparse pattern-based potentials.}
We can generalize the SMR approach to the case of \emph{robust sparse high-order potentials} using the ideas of \citet{kohli2009robust} and \citet{Rother09highorder}. Robust sparse high-order potentials in addition to rewarding the specific labelings $\vec{d} \in \mathcal{D}_C$ also reward labelings where the deviation from $\vec{d} \in \mathcal{D}_C$ is small.
We can formulate such a term for hyperedge $C$ in the following way:
\begin{equation}
\label{eq::patternBasedRobust}
\theta_C(\vec{x}_C) = \!\!\sum_{\vec{d} \in \mathcal{D}_C} \min \biggl( 0, \: \hat{\theta}_{C,\vec{d}} + \!\!\sum_{\ell \in C} w^C_\ell [x_\ell \neq d_\ell] \biggr)
\end{equation}
or equivalently in terms of indicator variables $\vec{y}$
\begin{equation}
\label{eq::patternBasedRobustBinary}
\theta_C(\vec{y}_C) = \!\!\sum_{\vec{d} \in \mathcal{D}_C} \min \biggl( 0, \: \hat{\theta}_{C,\vec{d}} + \!\!\sum_{\ell \in C} w^C_\ell (1 - y_{\ell d_\ell}) \biggr).
\end{equation}
The inner min of~\eqref{eq::patternBasedRobustBinary} can be transformed to pairwise form via adding switching variables $z_{C,\vec{d}}$ (analogous to \eqref{eq::smr::EYZsparseHighOrder}):
\[
\min_{z_{C,\vec{d}} \in \{0, 1\}} \quad \hat{\theta}_{C,\vec{d}} z_{C,\vec{d}} + \sum_{\ell \in C} w^C_\ell z_{C,\vec{d}} (1 - y_{\ell d_\ell}).
\]
This function is submodular w.r.t. variables $\vec{y}$ and $\vec{z}$ if coefficients $w^C_\ell$ are nonnegative. Introducing Lagrange multipliers and maximizing the dual function is absolutely analogous to the case of non-robust sparse high-order potentials.

One particular instance of robust pattern-based potentials is the robust $\mathcal{P}^n$-Potts model~\cite{kohli2009robust}.
Here the cardinality of sets~$\mathcal{D}_C$ equals the number of the labels~$|\mathcal{P}|$ and all the selected labelings correspond to the uniform labelings of variables~$\vec{x}_C$:
\begin{equation}
\label{eq::robustPnPotts}
\theta_C(\vec{x}_C) = \sum_{p \in \mathcal{P}} \min \biggl( 0, -\gamma + \sum_{\ell \in C} \gamma / Q [x_\ell \neq p] \biggr).
\end{equation}
Here~$\gamma \geq 0$ and $Q\in (0,|C|]$ are the model parameters. $\gamma$ represents the reward that energy gains if all the variables~$\vec{x}_C$ take equal values. $Q$ stands for the number of variables that are allowed to deviate from the selected label~$p$ before the penalty is truncated.\!\footnote{In the original formulation of the robust $\mathcal{P}^n$-Potts model~\cite{kohli2009robust} the minimum operation was used instead of the outer sum in~\eqref{eq::robustPnPotts}. The two schemes are equivalent if the cost of the deviation from the selected labelings if high enough, i.e. $Q < 0.5 |C|$.
The substitution of minimum with summation in similar context was used by~\citet{Rother09highorder}.
}

\paragraph{Positive values of~$\hat{\theta}_C$.}
Now we describe a simple way to generalize our approach to the case of general high-order potentials, i.e. potentials of form~\eqref{eq::patternBased} but with values~$\hat{\theta}_{C,\vec{d}}$ allowed to be positive.

First, note that constraints~\eqref{eq::smd::discretizationContraints} and~\eqref{eq::smd::consistencyConstraints} imply that for each hyperedge
$
C \in \mathcal{C}
$
equality
$
\sum_{\vec{p}\in\mathcal{P}^{C}} \prod_{\ell \in C}y_{ \ell p_\ell} = 1
$
holds.
Because of this fact we can subtract the constant~$M_C = \max_{\vec{d} \in \mathcal{D}_C} \hat{\theta}_C(\vec{d})$ from all values of potential~$\theta_C$ simultaneously without changing the minimum of the problem, i.e.
$$
\min_{\vec{y} \in \eqref{eq::smd::discretizationContraints}, \eqref{eq::smd::consistencyConstraints}}\!\! E_I(\vec{y})
=
\!\!\!\!\!\!\min_{\vec{y} \in \eqref{eq::smd::discretizationContraints}, \eqref{eq::smd::consistencyConstraints}}
\sum_{C\in\mathcal{C}}\biggl(M_C + \!\!\sum_{\vec{p}\in\mathcal{P}^{C}} \!\!\tilde{\theta}_C(\vec{p})  \prod_{\ell \in C}y_{ \ell p_\ell}\biggr)
$$
\noindent where $\tilde{\theta}_C(\vec{p}) = \theta_C(\vec{p}) - M_C$.
We have all $\tilde{\theta}_C(\vec{p})$ non-positive and therefore can apply the SMR technique described at the beginning of section~\ref{sec::highOrderMRFs}. Note, that in this case the number of variables~$\vec{z}$ required to add within the SMR approach can be exponential in the arity of hyper edge~$C$. Therefore the SMR is practical when either the arity of all hyperedges  is small or when there exists a compact representation such as~\eqref{eq::patternBased} or~\eqref{eq::patternBasedRobust}.

\subsection{Linear global constraints \label{sec::globalTerms}}
Linear constraints on the indicator variables form  an important type of global constraints.
The following constraints are special case of this class:
\begin{align}
\label{eq::strictAreaConstr}
&\sum\nolimits_{j\in\mathcal{V}}y_{jp}=c,\; \; \mbox{strict class size constraints,}\\
\notag 
&\sum\nolimits_{j\in\mathcal{V}}y_{jp}\in[c_1,c_2],\; \; \mbox{soft class size constraints,}\\
\notag 
&\sum\nolimits_{j\in\mathcal{V}}y_{jp}I_j=\sum\nolimits_{j\in\mathcal{V}}y_{jq}I_j,\; \; \mbox{flux equalities,}\\
\notag 
&
\sum\nolimits_{j\in\mathcal{V}}y_{jp}I_j=\vec{\mu}\sum\nolimits_{i\in\mathcal{V}}y_{ip},\; \; \mbox{mean,}\\
\notag 
&
\sum\nolimits_{j\in\mathcal{V}}\!y_{jp}(I_j\!-\!\vec\mu)(I_j\!-\!\vec\mu)^T\!=\!\sum\nolimits_{j\in\mathcal{V}}\!y_{ip}\Sigma,\;\; \mbox{variance.}
\end{align}
Here $I_j$ is the observable scalar or vector value associated with node~$j$, e.g. intensity or color of the corresponding image pixel.
There have been a number of works that separately consider constraints on total number of nodes that take a specific label (e.g. \cite{Kropotov10}).

The SMR~approach can easily handle general linear constraints such as
\begin{align}
\label{eq::globalConstr::equality}
\sum_{i\in\mathcal{V}} \sum_{p \in \mathcal{P}} w_{ip}^m y_{ip} &= c^m,\; \; m=1,\ldots,M, \\
\label{eq::globalConstr::inequality}
\sum_{i\in\mathcal{V}} \sum_{p \in \mathcal{P}} v_{ip}^k y_{ip} &\leq d^k,\; \; k=1,\ldots,K.
\end{align}

We introduce additional Lagrange multipliers $\vec{\xi} = \{ \xi_m \}_{m=1,\dots,M}$, $\vec{\pi} = \{ \pi_k \}_{k=1,\dots,K}$ for constraints \eqref{eq::globalConstr::equality}, \eqref{eq::globalConstr::inequality} and get the Lagrangian
\begin{align*}
L(\vec{y}, \vec{\lambda}, \vec{\xi}, \vec{\pi}) = &E_I(\vec{y})+ \sum_{i \in \mathcal{V}} \lambda_i \biggl( \sum_{p \in \mathcal{P}} y_{ip} -1 \biggr)
+
\\
&\sum_{m=1}^M \xi_m   \biggl( \sum_{i\in\mathcal{V}}  \sum_{p \in \mathcal{P}} w_{ip}^m y_{ip} - c^m \biggr)
+
\\
&\sum_{k=1}^K \pi_k   \biggl( \sum_{i\in\mathcal{V}} \sum_{p \in \mathcal{P}} v_{ip}^k y_{ip} - d^k \biggr)
\end{align*}
where inequality multipliers $\vec{\pi}$ are constrained to be nonnegative.
The Lagrangian is submodular w.r.t. indicator variables $\vec{y}$ and thus dual function
\begin{equation}
\label{eq::globalConstr::dualfunc}
D(\vec{\lambda}, \vec{\xi}, \vec{\pi}) = \min_{\vec{y}\in \eqref{eq::smd::discretizationContraints}} L(\vec{y}, \vec{\lambda}, \vec{\xi}, \vec{\pi})
\end{equation}
can be efficiently evaluated. The dual function is piecewise linear and concave w.r.t. all variables and thus can be treated similarly to the dual functions introduced earlier.

\section{Algorithms \label{sec::theorOptimization}}
In sec.~\ref{sec::optimization} we discuss the convex optimization methods that can be used to maximize the dual functions~$D(\vec{\lambda})$ constructed within the SMR approach.
In sec.~\ref{sec::marginals} we develop an alternative approach based on coordinate-wise ascent.

\subsection{Convex optimization \label{sec::optimization}}
We consider several optimization methods that might be used to maximize the SMR dual functions:
\begin{enumerate}
\item subgradient ascent methods~\cite{Komodakis10ddtrw};
\item bundle methods~\cite{Kappes12bundle};
\item non-smooth version of L-BFGS~\cite{Lewis08};
\item proximal methods~\cite{Zach12relax};
\item smoothing-based techniques~\cite{Savchynskyy11nesterov}, \cite{Zach12relax};
\item stochastic subgradient methods~\cite{Komodakis10ddtrw}.
\end{enumerate}

First, we begin with the methods that are not suited to SMR: proximal, smoothing-based and stochastic.
The first two groups of methods are not directly applicable for SMR because they require not only the black-box oracle that evaluates the function and computes the subgradient (in SMR this is done by solving the max-flow problem) but request additional computations: proximal methods~-- the prox-operator, log-sum-exp smoothing~\cite{Savchynskyy11nesterov}~-- solving the sum-product problem instead of max-sum, quadratic smoothing~\cite{Zach12relax} requires explicit computation of convex conjugates.\!\footnote{\citet{Zach12relax} directly construct the smooth approximation~$f_\varepsilon$ of a convex non-smooth function~$f$ by $f_\varepsilon = (f^* + \varepsilon\|\cdot\|^2_2/2)^*$ where $\varepsilon > 0$ is a smoothing parameter and $(\cdot)^*$ is the convex conjugate. In their case function~$f$ is a sum of summands of simple form for which analytical computation of convex conjugates is possible. In our case each summand is a min-cut LP, its conjugate is a max-flow LP. The addition of quadratic terms to $f^*$ makes $f_\varepsilon$ quadratic instead of a min-cut LP and thus efficient max-flow algorithms are no more applicable.}
We are not aware of possibilities of computing any of those in case of the SMR approach.

Stochastic techniques have proven themselves to be highly efficient for large-scale machine learning problems. \citet[sec.~13.3.2]{Bousquet12tradeoffs} show that stochastic methods perform well within the estimation-optimization tradeoff but are not so great as optimization methods for the empirical risk (i.e. the optimized function). In case of SMR we do not face the machine learning tradeoffs and thus stochastic subgradient is not suitable.

Further in this section we give a short review of methods of the three groups (subgradient, bundle, L-BFGS) that require only a black-box oracle that for an arbitrary value of dual variables $\vec{\lambda}$ computes the function value~$D(\vec{\lambda})$ and one subgradient~$\partial D(\vec{\lambda})$.

\textbf{Subgradient methods} at iteration~$n$ perform an update in the direction of the current subgradient
\[
\vec{\lambda}^{n+1} = \vec{\lambda}^{n} + \alpha^n \partial D(\vec{\lambda}^n).
\]
with a step size~$\alpha^n$.
\citet{Komodakis10ddtrw} and \citet{Kappes12bundle} suggest the following adaptive scheme of selecting the step size:
\begin{equation}
\label{eq::subgrAdaptiveStepSize}
\alpha^n = \gamma \frac{A^n - D(\vec{\lambda}^n)}{ \| \partial D(\vec{\lambda}^n) \|^2_2}
\end{equation}
where $A^n$ is the current estimate of the optimum (the best value of the primal solution found so far) and $\gamma$ is a parameter. Fig.~\ref{alg::subgradient} summarizes the overall method.

\begin{figure}[!tp]
\begin{algorithmic}[1]
\REQUIRE Lagrangian~$L(\vec{y}, \vec{\lambda})$~\eqref{eq::smd::lagr}, initialization~$\vec{\lambda}_0$;
\ENSURE $\vec{\lambda}^* = \argmax_{\vec{\lambda}} \min_{\vec{y}} L(\vec{y}, \vec{\lambda})$;
\STATE $\vec{\lambda} := \vec{\lambda}_0$; 
\REPEAT
    \STATE $\vec{y}^* := \argmin_{\vec{y}} L(\vec{y}, \vec{\lambda})$; \COMMENT{run max-flow}
    \STATE $D(\vec{\lambda}) = L(\vec{y}^*\!\!, \vec{\lambda})$; \COMMENT{evaluate the function}
    \FORALL{$i \in \mathcal{V}$}
        \STATE $g_i := \!\!\sum\limits_{p\in\mathcal{P}} y^*_{ip} - 1$;  \COMMENT{compute the subgradient}
    \ENDFOR
    \IF{ $\vec{g}=0$ }
        \RETURN{$\vec{\lambda}$}; \COMMENT{the certificate of optimality}
    \ELSE
        \STATE estimate the primal; \COMMENT{see sec.~\ref{sec::primalSolution}}
        \STATE choose stepsize $\alpha$; \COMMENT{e.g. using rule~\eqref{eq::subgrAdaptiveStepSize}}
        \STATE $\vec{\lambda} := \vec{\lambda} + \alpha \vec{g}$; \COMMENT{make the step}
    \ENDIF
\UNTIL{convergence criteria are met }
\RETURN{$\vec{\lambda}$};
\end{algorithmic}
\caption{Subgradient ascent algorithm to maximize the SMR dual~\eqref{eq::smd::dualFunc}. \label{alg::subgradient}}
\end{figure}

\textbf{Bundle methods} were analyzed by \citet{Kappes12bundle} in the context of MRF energy minimization. A bundle $\mathcal{B}$ is a collection of points~$\vec{\lambda}'$, values of dual function~$D(\vec{\lambda}')$, and sungradients~$\partial D(\vec{\lambda}')$. The next point is computed by the following rule:
\begin{equation}
\label{eq::bundleUpdateRule}
\vec{\lambda}^{n+1} = \argmax_{\vec{\lambda}} \; \left( \hat{D}(\vec{\lambda}) - \frac{w^n}{2} \| \vec{\lambda} - \bar{\vec{\lambda}}  \|_2^2 \right)
\end{equation}
where $\hat{D}(\vec{\lambda})$ is the upper bound on the dual function:
\[
\hat{D}(\vec{\lambda}) = \min_{ ( \vec{\lambda}^\prime, D(\vec{\lambda}^\prime), \partial D(\vec{\lambda}^\prime) ) \in \mathcal{B}} \bigl(  D(\vec{\lambda}^\prime) + \langle \partial D(\vec{\lambda}^\prime), \vec\lambda - \vec\lambda^\prime \rangle \bigr).
\]
Parameter~$w^n$ is intended to keep $\vec\lambda^{n+1}$ close to the current solution estimate~$\bar{\vec{\lambda}}$.
If the current step does not give a significant improvement than the \emph{null step} is performed: bundle~$\mathcal{B}$ is updated with $(\vec{\lambda}^{n+1}, D(\vec{\lambda}^{n+1}), \partial D(\vec{\lambda}^{n+1}))$. Otherwise the \emph{serious step} is performed: in addition to updating the bundle the current estimate $\bar{\vec{\lambda}}$ is replaced by $\vec{\lambda}^{n+1}$. \citet{Kappes12bundle} suggest to use the relative improvement of $D( \vec{\lambda}^{n+1} )$ and $\hat{D}( \vec{\lambda}^{n+1} )$ over $D( \bar{\vec{\lambda}} )$ (threshold $m_L$) to choose the type of the step. Inverse step size $w^n$ can be chosen adaptively if the previous step was serious:
\begin{equation}
\label{eq::bundleStepSize}
w^n = P_{[ w_{\min}, w_{\max} ]} \left( \gamma \frac{A^n - \max_{ k = 1,\dots,n } D(\vec{\lambda}^k) }{ \| \partial D(\vec{\lambda}^n) \|_2 }  \right)^{-1}
\end{equation}
where $P_{[ w_{\min}, w_{\max} ]}$ is the projection on the segment. In case of the null step $w^n = w^{n-1}$.
In summary, the bundle method has the following parameters: $\gamma$, $w_{\min}$, $w_{\max}$, $m_L$, and maximum size of the bundle~-- $b_s$. Note that the update~\eqref{eq::bundleUpdateRule} implicitly estimates both the direction~$\vec{g}^n$ and the step size~$\alpha^n$.

Another algorithm analyzed by \citet{Kappes12bundle} is the \emph{aggregated bundle} method proposed by \citet{Kiwiel83} with the same step-size rule~\eqref{eq::bundleStepSize}.
The method has the following parameters: $\gamma$, $w_{\min}$, $w_{\max}$, and a threshold to choose null or serious step~-- $m_r$. Please see works \cite{Kappes12bundle}, \cite{Kiwiel83} for details.
We have also tried a combined strategy: LMBM method\footnote{\url{http://napsu.karmitsa.fi/lmbm/lmbmu/lmbm-mex.tar.gz}}\!~\cite{Haarala07} where we used only one non-default parameter: the maximum size of a bundle~$b_s$.

\textbf{L-BFGS methods} choose the direction of the update using the history of function evaluations:
$S^n \partial D(\vec{\lambda}^{n})$
where $S^n$ is a negative semi-definite estimate of the inverse Hessian at $\vec{\lambda}^{n}$. The step size~$\alpha^n$ is typically chosen via approximate one-dimensional optimization, a.k.a. line search.
\citet{Lewis08} proposed a specialized version of L-BFGS for non-smooth functions implemented in HANSO library.\!\footnote{\url{http://www.cs.nyu.edu/overton/software/hanso/}} In this code we varied only one parameter: the maximum rank of Hessian estimator~$h_r$.

\subsection{Coordinate-wise optimization \label{sec::marginals}}
\begin{figure}[!tp]
\begin{algorithmic}[1]
\REQUIRE Lagrangian~$L(\vec{y}, \vec{\lambda})$~\eqref{eq::smd::lagr}, initialization~$\vec{\lambda}_0$;
\ENSURE coordinate-wise maximum of $D(\vec{\lambda})$;
\STATE $\vec{\lambda} := \vec{\lambda}_0$;
\REPEAT
    \STATE $\text{converged} := \text{true}$;
    \FORALL{$j \in \mathcal{V}$}
        \FORALL{$p \in \mathcal{P}$}
            \STATE compute $MM_{jp,0}  L(\vec{y}, \vec{\lambda})$, $MM_{jp,1}  L(\vec{y}, \vec{\lambda})$;
            \STATE $\delta_p^j := MM_{jp,0} \; L(\vec{y}, \vec{\lambda}) - MM_{jp,1} \; L(\vec{y}, \vec{\lambda})$;
        \ENDFOR
        \STATE find~$\delta_{(1)}^j$ and $\delta_{(2)}^j$;
        \IF{ $\delta_{(1)}^j $ and $\delta_{(2)}^j$ of the same sign}
            \STATE $\lambda_j := \lambda_j + 0.5(\delta_{(1)}^j + \delta_{(2)}^j)$;
            \STATE $\text{converged} := \text{false}$;
        \ENDIF
    \ENDFOR
\UNTIL{ $\text{converged} \neq \text{true}$};
\RETURN{$\vec{\lambda}$};
\end{algorithmic}
\caption{Coordinate ascent algorithm for maximizing the dual~$D(\vec{\lambda})$~\eqref{eq::smd::dualFunc} in case of pairwise associative potentials. \label{alg::minMarginals}}
\end{figure}

An alternative approach to maximize the dual function~$D(\vec{\lambda})$ consists in selecting a group of variables and finding the maximum w.r.t them. This strategy was used in max-sum diffusion~\cite{Werner07} and MPLP~\cite{Globerson07mplp} algorithms.
In this section we derive an analogous algorithm for the SMR in case of pairwise associative potentials.
The algorithm is based on the concept of min-marginal averaging used in~\cite{Wainwright05trw}, \cite{Kolmogorov06trws} to optimize the TRW dual function.

Consider some value~$\vec{\lambda}^{old}$ of dual variables~$\vec{\lambda}$ and a selected node~$j \in \mathcal{V}$. We will show how to construct a new value~$\vec{\lambda}^{new}$ such that~$D(\vec{\lambda}^{new}) \geq D(\vec{\lambda}^{old})$. The new value will differ from the old value only in one coordinate~$\lambda_j$, i.e. $\lambda_{i}^{new} = \lambda_{i}^{old}$, $i \neq j$.

Consider the \emph{unary min-marginals} of Lagrangian~\eqref{eq::smd::lagr}
$$
MM_{jp,k} L\left(\vec{y}, \vec{\lambda}^{old}\right)=\min_{\vec{y} \in \eqref{eq::smd::discretizationContraints},\; y_{jp} = k} L(\vec{y}, \vec{\lambda}^{old})
$$
where $k \in \{0,1\}$,  $p \in \mathcal{P}$.
The min-marginals can be computed efficiently using the dynamic graphcut framework~\cite{Kohli07dynamic}.
Denote the differences of min-marginals for labels~$0$ and~$1$ with~$\delta_p^j$:
$
\delta_{p}^j = MM_{jp,0} L\left(\vec{y}, \vec{\lambda}^{old}\right) - MM_{jp,1} L(\vec{y}, \vec{\lambda}^{old})
$. Let~$\delta_{(1)}^j$ be the maximum of~$\delta_{p}^j$,  $\delta_{(2)}^j$~-- the next maximum; $p^{(1)}_j$ and~$p^{(2)}_j$~-- the corresponding indices: $p^{(1)}_j = \argmax_{p \in \mathcal{P}} \delta_{p}^j$, $\delta_{(1)}^j = \delta_{p^{(1)}_j}^j$, $p^{(2)}_j = \argmax_{p \in \mathcal{P} \setminus \{ p^{(1)}_j \}} \delta_{p}^j$, $\delta_{(2)}^j = \delta_{p^{(2)}_j}^j$.
Let us construct the new value~$\lambda^{new}_j$ of dual variable~$\lambda_j$ with the following rule:
\begin{equation}
\label{eq::minMarginalUpdate}
\lambda_{j}^{new} = \lambda_{j}^{old} + \Delta_j
\end{equation}
where $\Delta_j$ is a number from segment~$\left[ \delta_{(2)}^j, \delta_{(1)}^j \right]$.

\begin{theorem}\label{theor::marginAveraging}
Rule~\eqref{eq::minMarginalUpdate} when applied to a pairwise energy with associative pairwise potentials delivers the maximum of function $D(\vec{\lambda})$ w.r.t. variable~$\lambda_j$.
\end{theorem}

The proof relies on the intuition that rule~\eqref{eq::minMarginalUpdate} directly assigns the min-marginals $MM_{jp,k} L\left(\vec{y}, \vec{\lambda}\right)$, $p \in \mathcal{P}$, $k\in \{0,1\}$ such values that only one subproblem~$p$ has $y_{j,p}$ equal to~$1$.
We provide the full proof in the supplementary material.

Theorem~\ref{theor::marginAveraging} allows us to construct the coordinate-ascent algorithm (fig.~\ref{alg::minMarginals}). This result cannot be extended to the cases of general pairwise or high-order potentials. We are not aware of analytical updates like~\eqref{eq::minMarginalUpdate} for these cases. The update w.r.t. single variable can be computed numerically
but the resulting algorithm would be very slow.

\subsection{Getting a primal solution \label{sec::primalSolution}}
Obtaining a primal solution given some value of the dual variables is an important practical issue related to the Lagrangian relaxation framework.
There are two different tasks: obtaining the feasible primal point of the relaxation and obtaining the final discrete solutions.
The first task was studied by \citet{Savchynskyy11nesterov} and \citet{Zach12relax} for algorithms based on the smoothing of the dual function.
Recently, \citet{Savchynskyy12primalSolution} proposed a general approach that can be applied to SMR whenever a subgradient algorithm is used to maximize the dual.
The second aspect is usually resolved using some heuristics.
\citet{Osokin12hipot} propose an heuristic method to obtain the discrete primal solution within the SMR approach.
Their method is based on the idea of block-coordinate descent w.r.t. the primal variables and is similar to the analogous heuristics in TRW-S~\cite{Kolmogorov06trws}.

\section{Theoretical analysis \label{sec::theory}}
In this section we explore some properties of the maximum points of the dual functions and express their maximum values explicitly in the form of LPs. The latter allows us to reason about tightness of the obtained lower bounds and to provide analytical comparisons against other approaches.
\subsection{Definitions and theoretical properties}
Denote the possible optimal labelings of binary variable $y_{ip}$ associated with node~$i$ and label~$p$ given Lagrange multiplyers~$\vec{\lambda}$ with
\begin{equation*}
Z_{ip}(\vec{\lambda}) =
\bigl\{ z \mid \exists \vec{y}^{\vec{\lambda}} \in \Argmin_{\vec{y} \in \eqref{eq::smd::discretizationContraints} } L(\vec{y}, \vec{\lambda})  \; : \; y^{\vec{\lambda}}_{ip}=z  \bigr\}
\end{equation*}
where~$L(\vec{y}, \vec{\lambda})$ is the Lagrangian constructed within the SMR approach~\eqref{eq::smd::lagr}.

\begin{definition}
Point $\vec{\lambda}$ satisfies the \emph{strong agreement condition} if for all nodes~$i$, for one label $p$ set~$Z_{ip}(\vec{\lambda})$ equals~$\{1\}$ and for the other labels it equals~$\{0\}$:
\[
\forall i\!\in\!\mathcal{V} \; \exists!p\!\in\!\mathcal{P}: \; Z_{ip}(\vec{\lambda}) = \{1\}, \; \forall q \neq p \; Z_{iq}(\vec{\lambda})=\{0\}.
\]
\end{definition}
This definition implies that for a strong agreement point $\vec{\lambda}$ there exists the unique binary labeling $\vec{y}$ consistent with all sets $Z_{ip}(\vec{\lambda})$ ($y_{ip} \in Z_{ip}(\vec{\lambda})$) and with consistency constraints~\eqref{eq::smd::consistencyConstraints}. Consequently, such $\vec{y}$ defines labeling $\vec{x} \in \mathcal{P}^\mathcal{V}$ that delivers the global minimum of the initial energy~\eqref{eq::notation::energyMultilabel}.


\begin{definition}
Point~$\vec{\lambda}$ satisfies the \emph{weak agreement condition} if for all nodes~$i\in \mathcal{V}$ both statements hold true:
\begin{enumerate}
\item $\exists p \in \mathcal{P}:\; 1 \in Z_{ip}(\vec{\lambda})$.
\item $\forall p \in \mathcal{P}:\; Z_{ip}(\vec{\lambda})=\{1\} \Rightarrow \forall q\ne p, \; 0 \in Z_{jq}(\vec{\lambda})$.
\end{enumerate}
\end{definition}

This definition means that for a weak agreement point~$\vec{\lambda}$ there exists a binary labeling $\vec{y}$ consistent with all sets $Z_{ip}(\vec{\lambda})$ and consistency constraints~\eqref{eq::smd::consistencyConstraints}.
It is easy to see that the strong agreement condition implies the weak agreement condition.

Denote a maximum point of dual function~$D(\vec{\lambda})$~\eqref{eq::smd::dualFunc} by $\vec{\lambda}^*$ and an assignment of the primal variables that minimize Lagrangian~$L(\vec{y}, \vec{\lambda}^*)$ by $\vec{y}^*$. We prove the following property of~$\vec{\lambda}^*$ in the supplementary material.
\begin{theorem}
\label{theor::wac}
A maximum point~$\vec{\lambda}^*$ of the dual function~$D(\vec{\lambda})$ satisfies the weak agreement condition.
\end{theorem}

In general the three situations are possible at the global maximum of the dual function of SMR (see fig.~\ref{fig::roofs} for illustrations):
\begin{enumerate}
\item The strong agreement holds, binary labeling~$\vec{y}^*$  satisfying consistency constraints~\eqref{eq::smd::consistencyConstraints} (thus labeling $\vec{x}^*$ as well) can be easily reconstructed, there is no integrality gap.
\item Some (maybe all) sets $Z_{ip}(\vec{\lambda}^*)$ consist of multiple elements, $\vec{x}^*$ consistent with~\eqref{eq::smd::consistencyConstraints} is non-trivial to reconstruct, but there is still no integrality gap.
\item Some sets $Z_{ip}(\vec{\lambda}^*)$ consist of multiple elements, labeling~$\vec{y}^*$ that is both consistent with~\eqref{eq::smd::consistencyConstraints} and is the minimum of the Lagrangian does not exist, the integrality gap is non-zero.
\end{enumerate}

In situations~1 and~2 there exists a binary labeling $\vec{y}^*$ that simultaneously satisfies consistency constraints~\eqref{eq::smd::consistencyConstraints} and delivers the minimum of the Lagrangian~$L(\vec{y}, \vec{\lambda}^*)$. Labeling~$\vec{y}^*$ defines the horizontal hyperplane that contains point~$(\vec{\lambda}^*, D(\vec{\lambda}^*))$.

In situation~1 labeling~$\vec{y}^*$ is easy to find because the subgradient defined by~\eqref{eq::dualSubgradient} equals zero.
Situation~2 can be trickier but in certain cases might be resolved exactly.
Some examples are given in \cite{Osokin11cvpr}.

\begin{figure}[t]
\centering
 \begin{tabular}{@{}l@{\;\;}l@{\;\;}l}
 {\footnotesize $D(\vec{\lambda})$} &
 {\footnotesize $D(\vec{\lambda})$} &
 {\footnotesize $D(\vec{\lambda})$} \\
 \includegraphics[width=2.3cm]{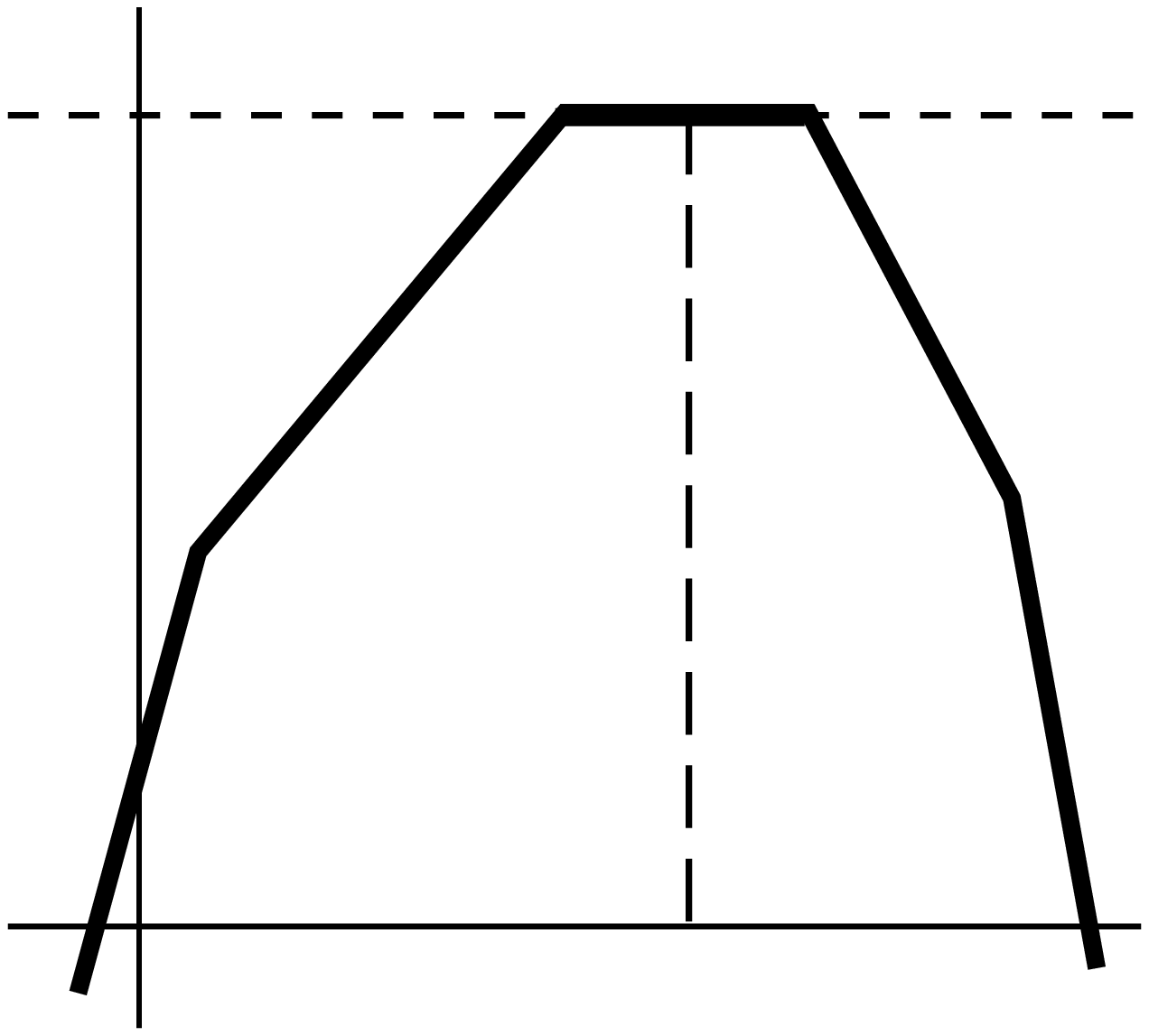} &
 \includegraphics[width=2.3cm]{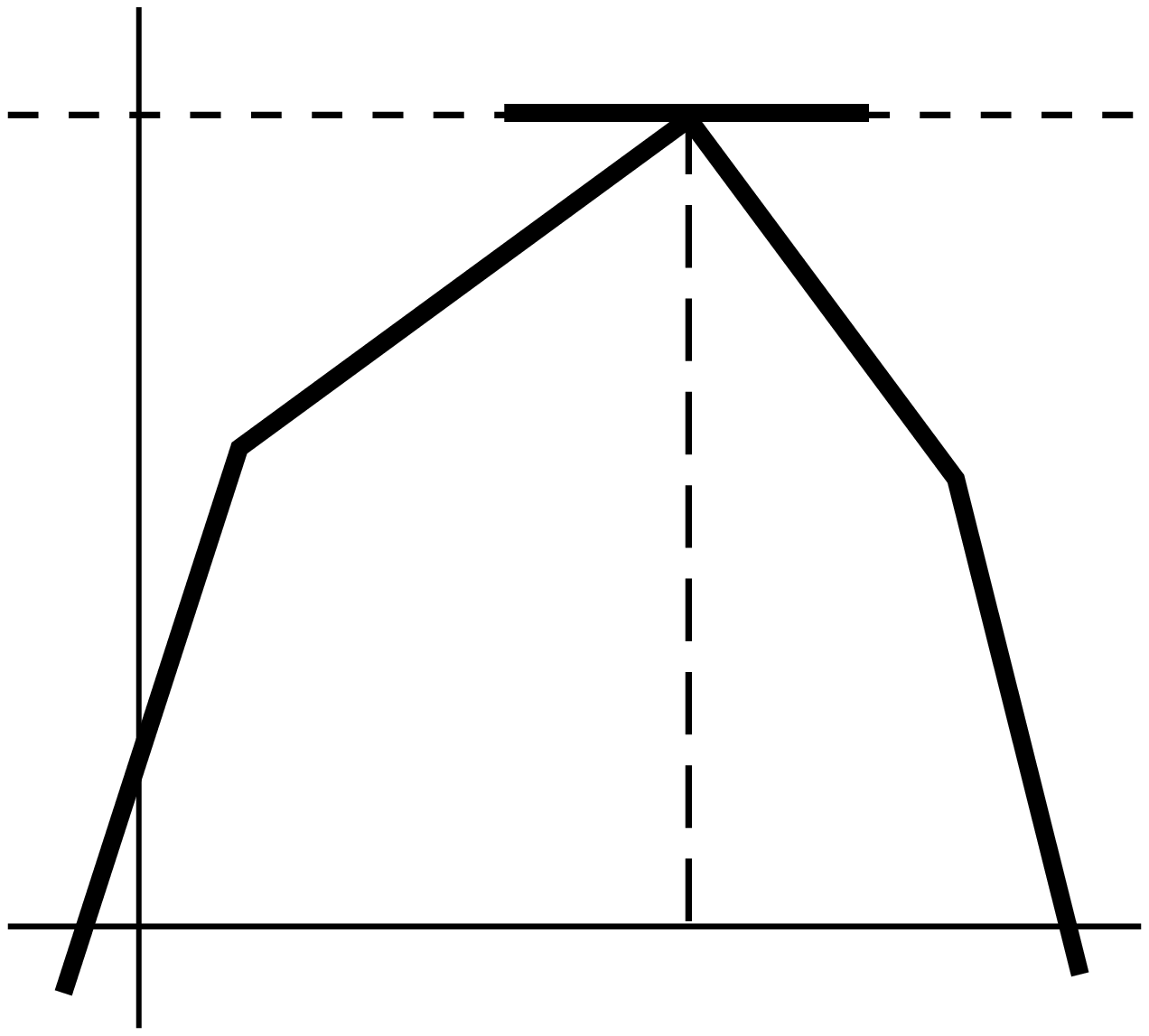} &
 \includegraphics[width=2.3cm]{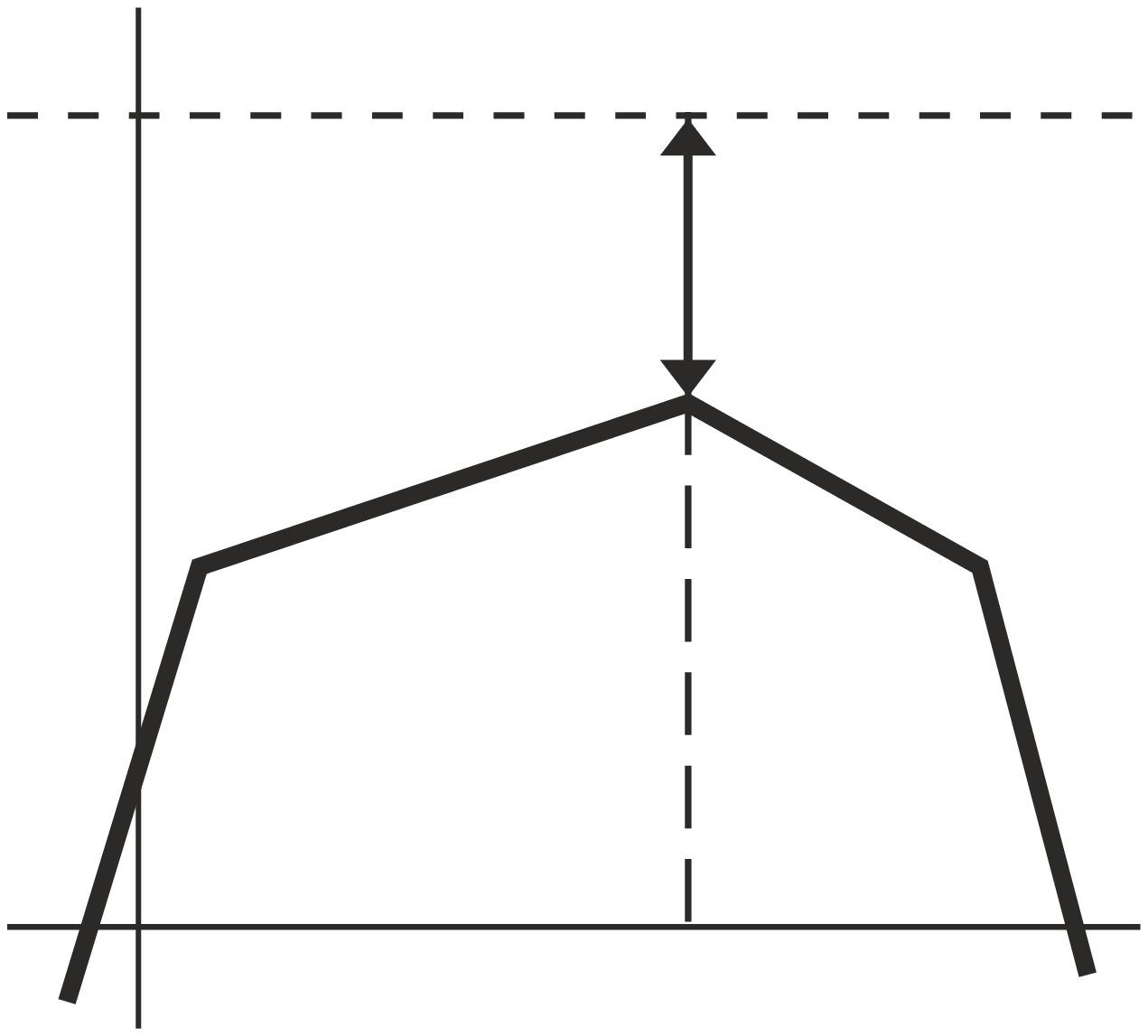} \vspace{-2.1cm}\\
 & &\;\;\;\;\;\;\;\;\;\;\;\;\;\;\;\parbox{1.3cm}{\begin{center}{\scriptsize integrality}\\[-2mm]  {\scriptsize gap}\end{center}}\vspace{1.0cm}\\
 \;\;\;\;\;\;\;\;\;\;\;\;\;{\footnotesize $\vec{\lambda}^*$}\;\;\;\;{\footnotesize $\vec{\lambda}$}&
 \;\;\;\;\;\;\;\;\;\;\;\;\;{\footnotesize $\vec{\lambda}^*$}\;\;\;\;{\footnotesize $\vec{\lambda}$}&
 \;\;\;\;\;\;\;\;\;\;\;\;\;{\footnotesize $\vec{\lambda}^*$}\;\;\;\;{\footnotesize $\vec{\lambda}$}\\[-0.2cm]
 (1) & (2) & (3) \\
\end{tabular}
\caption {The three possible cases at the maximum of the SMR dual function. The optimal energy value is shown by the horizontal dotted line. Solid lines show faces of the hypograph of the dual function~$D(\vec{\lambda})$.
\label{fig::roofs}}
\end{figure}


In practice the convex optimization methods described in sec.~\ref{sec::optimization} typically
converge to the point that satisfies the strong agreement condition rather quickly when there is no integrality gap. In case of the non-zero integrality gap all the methods start to oscillate around the optimum and do not provide a point satisfying even the weak agreement condition.
The coordinate-wise algorithm (fig.~\ref{alg::minMarginals}) often converges to points with low values of the dual function (there exist multiple coordinate-wise maximums of the dual function) but allows to obtain a point that satisfies the weak agreement condition. This result is summarized in theorem~\ref{theor::minMarginalWeakAgreement} proven in the supplementary material.
\begin{theorem}
\label{theor::minMarginalWeakAgreement}
Coordinate ascent algorithm (fig.~\ref{alg::minMarginals}) when applied to a pairwise energy with associative pairwise potentials returns a point that satisfies the weak agreement condition.
\end{theorem}

\subsection{Tightness of the lower bounds \label{sec::tightness}}

\subsubsection{Pairwise associative MRFs}
\citet{Osokin11cvpr} show that in the case of pairwise associative MRFs maximizing Lagrangian dual~\eqref{eq::smd::dualFunc} is equivalent to the standard LP relaxation of pairwise energy~\eqref{eq::notation::energyMultilabelPairwise}, thus the SMR lower bound equals the ones of competitive methods such as~\cite{Komodakis10ddtrw}, \cite{Kappes12bundle}, \cite{Savchynskyy11nesterov} and in general is tighter than the lower bounds of TRW-S~\cite{Kolmogorov06trws} and MPLP~\cite{Globerson07mplp} (TRW-S and MPLP can get stuck at bad coordinate-wise maximums).
In addition \citet{Osokin11cvpr} show that in the case of pairwise associative MRFs the standard LP relaxation is equivalent to Kleinberg-Tardos~(KT) relaxation of a uniform metric labeling problem~\cite{Kleinberg02}.
In this section we reformulate the most relevant result for the sake of completeness. The proof will be given later as a corollary of our more general result for high-order MRFs (theorems \ref{theor::smrTightness} and \ref{theor::smrTightnessPermutedPotts}).

\begin{theorem}\label{theor::smdLB1}
In case of pairwise MRFs with associative pairwise potentials~\eqref{eq::notation::energyMultilabelPairwise}, \eqref{eq::smd::Potts} the maximum of Lagrangian dual~\eqref{eq::smd::dualFunc} equals the minimum value of the standard LP relaxation~\eqref{eq::notation::energyLinear}, \eqref{eq::notation::localGConsistensy}, \eqref{eq::notation::localGMarginalization} of the discrete energy minimization problem.
\end{theorem}

\subsubsection{High-order MRFs}
In this section we state the general result, compare it to known LP relaxations of high-order terms and further elaborate on one important special case.

The following theorem explicitly defines the LP relaxation of energy~\eqref{eq::notation::energyMultilabel} that is solved by SMR.
\begin{theorem}
 \label{theor::smrTightness}
 In case of high-order MRF with sparse pattern-based potentials~\eqref{eq::patternBased} the maximum of the SMR dual function~$D(\vec{\lambda})$~\eqref{eq::smr::dualHighOrder} equals  the minimum value of the following linear program:
 \begin{align}
\label{eq::smrLpInitialTarget}
\min_{\vec{y}}\; &   \sum_{C\in\mathcal{C}}\sum_{\vec{d} \in \mathcal{D}_C}
\hat{\theta}_{C,\vec{d}}\; y_{C,\vec{d}}   \\ 
\label{eq::smrLpInitialIntegrality1}
\st \; &y_{ip},y_{C,\vec{d}} \in[0,1], \; \forall i\!\in\!\mathcal{V}, \: \forall p\!\in\!\mathcal{P}, \: \forall C\!\in\!\mathcal{C}, \: \forall \vec{d}\!\in\!\mathcal{D}_C,\\
\label{eq::smrLpInitialConstr1}
& y_{C,\vec{d}} \leq y_{\ell d_\ell}, \quad \forall C\in\mathcal{C}, \; \forall\vec{d}\in\mathcal{D}_C, \; \forall \ell \in C,\\
\label{eq::smrLpInitialConsistensy}
&\sum_{p \in \mathcal{P} } y_{ip}=1, \quad \forall i\in\mathcal{V}.
\end{align}
\end{theorem}
\begin{proof}
We start the proof with the two lemmas.
\begin{lemma}
\label{theor::lpHighOrderLemma}
Let~$y_1, \dots, y_K$ be real numbers belonging to the segment~$[0, 1]$, $K > 1$.
Consider the following linear program:
\begin{align}
\label{eq::lemmaHighOrderLp:target}
\min_{z, z_1, \dots, z_K \in [0, 1]} \quad&  z(K - 1) - \sum_{k = 1}^K z_k \\
\label{eq::lemmaHighOrderLp:constraint}
\st \qquad\;\; &
z_k \leq z, \; z_k \leq y_k.
\end{align}
Point~$z = z_1 = \dots = z_K = \min_{k = 1, \dots, K} y_k$ delivers the minimum of LP~\eqref{eq::lemmaHighOrderLp:target}-\eqref{eq::lemmaHighOrderLp:constraint}.
\end{lemma}
Lemma~\ref{theor::lpHighOrderLemma} is proven in the supplementary material.
\begin{lemma}
\label{theor::lagrangianDual}
Consider the convex optimization problem
\begin{align}
\label{eq::lemmaKKT:target}
\min_{\vec{y} \in \mathcal{Y}} \quad & f(\vec{y}), \\
\label{eq::lemmaKKT:constr}
\st \quad
& A\vec{y} = \vec{b},
\end{align}
where $\mathcal{Y} \subseteq \mathbb{R}^n$ is a convex set, $A$ is a real-valued matrix of size~$n \times m$, $\vec{b} \in \mathbb{R}^m$, $f: \mathbb{R}^n \to \mathbb{R}$ is a convex function. Let set~$\mathcal{Y}$ contain a non-empty interior and the solution of problem~\eqref{eq::lemmaKKT:target}-\eqref{eq::lemmaKKT:constr} be finite.

Then the value of the solution of~\eqref{eq::lemmaKKT:target}-\eqref{eq::lemmaKKT:constr} equals
$
\max\limits_{\vec{\lambda} \in  \mathbb{R}^m }
\min\limits_{\vec{y} \in \mathcal{Y}}
L(\vec{y}, \vec{\lambda})
$
where~$
L(\vec{y}, \vec{\lambda}) =
f(\vec{y}) + \vec{\lambda}^T (A \vec{y}-\vec{b}).
$
\end{lemma}
Lemma~\ref{theor::lagrangianDual} is a direct corollary of the strong duality theorem with linearity constraint qualification~\cite[Th.~6.4.2]{bertsekas2003book}.

Let us finish the proof of theorem~\ref{theor::smrTightness}.
All values of the potentials are non-positive~$\hat{\theta}_{C,\vec{d}} \leq 0$ so lemma~\ref{theor::lpHighOrderLemma} implies that the LP~\eqref{eq::smrLpInitialTarget}-\eqref{eq::smrLpInitialConsistensy} is equivalent to the following LP:\\[-0.3cm]
\begin{align}
\label{eq::smrLpTarget}
\min_{\vec{y}, \vec{z}}\; &-\!\sum_{C\in\mathcal{C}}\sum_{\vec{d} \in \mathcal{D}_C}
\!\hat{\theta}_{C,\vec{d}}\biggl((|C|-1) z_{C,\vec{d}} - \!\sum_{\ell \in C }\! z_{C,\vec{d}}^{\ell}\biggr)\\
\label{eq::smrLpIntegrality}
\st\; &z_{C,\vec{d}}^\ell \in [0, 1], \quad\forall C\!\in\!\mathcal{C}, \; \forall\vec{d}\!\in\!\mathcal{D}_C, \; \forall \ell\!\in\!C, \\
 \notag & y_{ip},z_{C,\vec{d}} \in[0,1], \; \forall i\!\in\!\mathcal{V}, \: \forall p\!\in\!\mathcal{P}, \: \forall C\!\in\!\mathcal{C}, \: \forall \vec{d}\!\in\!\mathcal{D}_C,\\
\label{eq::smrLpConstr} &z_{C,\vec{d}}^{\ell} \leq z_{C,\vec{d}}, \: z_{C,\vec{d}}^{\ell} \le y_{\ell d_\ell},
  C\!\in\!\mathcal{C}, \: \forall\vec{d}\!\in\!\mathcal{D}_C, \: \forall \ell\!\in\!C,\\[-0.1cm]
\label{eq::smrLpConsistensy}
&\sum_{p \in \mathcal{P} } y_{ip}=1, \quad\quad\: \forall i\!\in\!\mathcal{V}.
\end{align}
Denote the target function~\eqref{eq::smrLpTarget} by $Q(\vec{y}, \vec{z})$. Let\\[-0.3cm]
\begin{equation*}
R(\vec{\lambda})=\min_{\vec{y}, \vec{z} \in \eqref{eq::smrLpIntegrality}-\eqref{eq::smrLpConstr}} \biggl(Q(\vec{y}, \vec{z}) +
\sum_{i\in\mathcal{V}}\lambda_i\biggl(\sum_{p \in \mathcal{P}} y_{ip}-1\biggr)\biggr).
\end{equation*}
Task~\eqref{eq::smrLpTarget}, \eqref{eq::smrLpIntegrality}-\eqref{eq::smrLpConstr} is equivalent to the standard LP-relaxation of function~\eqref{eq::smr::EYZsparseHighOrder} of binary variables~$\vec{y}$ and $\vec{z}$ (see lemma~3 in the supplementary material).  As~$\hat{\theta}_{C,\vec{d}} \leq 0$, function~\eqref{eq::smr::EYZsparseHighOrder} is submodular and thus its LP-relaxation is tight~\cite[Th.~3]{Kolmogorov05trw}, which means that~$R(\vec{\lambda})=D(\vec{\lambda})$ for an arbitrary~$\vec{\lambda}$.
Function~$f(\vec{y}) = \min_{\vec{z} \in \eqref{eq::smrLpIntegrality}-\eqref{eq::smrLpConstr}} Q(\vec{y}, \vec{z})$,  set~$\mathcal{Y}$ defined by~$y_{ip}\in[0,1]$, matrix~$A$ and vector~$\vec{b}$ defined by~\eqref{eq::smrLpConsistensy} satisfy the conditions of lemma~\ref{theor::lagrangianDual} which implies that the value of solution of LP~\eqref{eq::smrLpTarget}-\eqref{eq::smrLpConsistensy} equals~$R(\vec{\lambda}^*)$, where~$\vec{\lambda}^* =\argmax_{\vec{\lambda}} R(\vec{\lambda})$. The proof is finished by equality~$R(\vec{\lambda}^*) = D(\vec{\lambda}^*)$.
\end{proof}

\begin{corollary}
\label{theor::smrTightnessPairwise}
For the case of pairwise energy~\eqref{eq::notation::energyMultilabelPairwise} with $\theta_{ij,pq} \leq 0$ the SMR returns the lower bound equal to the solution of the following LP-relaxation:\\[-0.3cm]
\begin{align}
\label{eq::LP1:pairwise}
\min_{\vec{y}}\; &\sum_{i \in \mathcal{V}} \sum_{p \in \mathcal{P}} \theta_{ip} y_{ip} + \sum_{ \{i, j\} \in \mathcal{E}} \sum_{p,q \in \mathcal{P}} \theta_{ij, pq} y_{ij, pq}\\
\notag
\st\; &y_{ip}\in[0,1],  \quad \forall i\in\mathcal{V}, \; \forall p\in\mathcal{P},\\
\notag
& y_{ij, pq} \in[0,1],  \quad \forall \{i,j\} \in \mathcal{E}, \; \forall p,q \in \mathcal{P}, \\
\label{eq::LP2cons:pairwise} & y_{ij, pq} \leq y_{i p}, \: y_{ij, pq} \leq y_{j q},
 \; \forall \{i,j\}\!\in\!\mathcal{E},  \forall p,q\!\in\!\mathcal{P},\\
\label{eq::LP3cons:pairwise}
&\sum\nolimits_{p \in \mathcal{P} } y_{ip}=1, \quad \forall i\in\mathcal{V}.
\end{align}
\end{corollary}

\paragraph{Permuted $\mathcal{P}^n$-Potts.}
\citet{Komodakis09highorder} formulate a generally tighter relaxation than~\eqref{eq::smrLpInitialTarget}-\eqref{eq::smrLpInitialConsistensy} by adding constraints responsible for marginalization of high-order potentials w.r.t all but one variable:
\begin{equation}
\label{eq::LP4cons}
\sum_{\vec{p}\in\mathcal{P}^{C}: \;p_\ell=p_0} \!\!\!\!\!y_{C,\vec{p}}  =  y_{\ell p_0},\; \forall C\in\mathcal{C}, \: \forall p_0\in\mathcal{P}, \: \forall \ell \in C.
\end{equation}
In this paragraph we show that in certain cases the relaxation of~\cite{Komodakis09highorder} is equivalent
to~\eqref{eq::smrLpInitialTarget}-\eqref{eq::smrLpInitialConsistensy}, i.e. constraints~\eqref{eq::LP4cons} are redundant.
\begin{definition} $\!$Potential $\theta_C$ is called permuted $\mathcal{P}^n$-Potts if\\[-0.3cm]
$$
\forall C\in\mathcal{C} \;\; \forall \vec{d}',\vec{d}''\in\mathcal{D}_C \;\; \forall i\in C  : \quad \vec{d}' \neq \vec{d}''\Rightarrow d'_i\ne d''_i.
$$
\end{definition}

In permuted $\mathcal{P}^n$-Potts potentials all preferable configurations differ from each other in all variables. The $\mathcal{P}^n$-Potts potential described by~\citet{Kohli08potts} is a special case.
\begin{theorem}
\label{theor::smrTightnessPermutedPotts}
If all higher-order potentials are permuted $\mathcal{P}^n$-Potts, then the maximum of dual function~\eqref{eq::smr::dualHighOrder} is equal to the value of the solution of LP~\eqref{eq::smrLpInitialTarget}-\eqref{eq::smrLpInitialConsistensy}, \eqref{eq::LP4cons}.
\end{theorem}

\begin{proof}
First, all values~$\hat{\theta}_{C,\vec{d}}$ are non-positive so variables~$y_{C,\vec{p}}$ take maximal possible values. This implies that equality in~\eqref{eq::LP4cons} can be substituted with inequality.
Second, for a permuted $\mathcal{P}^n$-Potts potential $\theta_C$ for each $\ell \in C$ and for each $p_0\in\mathcal{P}$ there exists no more than one labeling $\vec{d} \in \mathcal{D}_C$ such that $d_\ell = p_0$ which means that all the sums in~\eqref{eq::LP4cons} can be substituted by the single summands. These two observations prove that~\eqref{eq::LP4cons} is equivalent to~\eqref{eq::smrLpInitialConstr1}.
\end{proof}

The LP result for the associative pairwise MRFs (theorem~\ref{theor::smdLB1}) immediately follows from theorem~\ref{theor::smrTightnessPermutedPotts} and corollary~\ref{theor::smrTightnessPairwise} because all associative pairwise potentials~\eqref{eq::smd::Potts} are permuted $\mathcal{P}^n$-Potts and condition \eqref{eq::LP4cons} reduces to marginalization constraints~\eqref{eq::notation::localGMarginalization}.

\subsubsection{Global linear constraints}

\begin{theorem}
\label{theor::globalConstrLB}
The maximum of the dual function~$D(\vec{\lambda}, \vec{\xi}, \vec{\pi})$~\eqref{eq::globalConstr::dualfunc} (constructed for the minimization of pairwise energy~\eqref{eq::notation::energyMultilabelPairwise} under global linear constraints~\eqref{eq::globalConstr::equality} and~\eqref{eq::globalConstr::inequality}) under constraints $\pi_k \geq 0$ is equal to the value of the solution of LP relaxation~\eqref{eq::LP1:pairwise}-\eqref{eq::LP3cons:pairwise}  with addition of the global linear constraints~\eqref{eq::globalConstr::equality} and~\eqref{eq::globalConstr::inequality}.
\end{theorem}

The proof is similar to the proof of theorem~\ref{theor::smrTightness} but requires including extra terms in the Lagrangian. Theorem~\ref{theor::globalConstrLB} generalizes to the case of high-order potentials.

\section{Experimental evaluation \label{sec::experiments}}

\subsection{Pairwise associative MRFs \label{sec::exp::optimization}}

In this section we evaluate the performance of different optimization methods discussed in sec.~\ref{sec::optimization} and compare the SMR approach against the dual decomposition on trees used in~\cite{Komodakis10ddtrw}, \cite{Kappes12bundle} (DD~TRW) and TRW-S method~\cite{Kolmogorov06trws}.

\paragraph{Dataset.} To perform our tests we use the MRF instances created by~\citet{Alahari10}.\!\footnote{\url{http://www.di.ens.fr/~alahari/data/pami10data.tgz}} We work with instances of two types: stereo (four MRFs) and semantic image segmentation (five MRFs). All of them are pairwise MRFs with 4-connectivity system and Potts pairwise potentials. The instances of these two groups are of significantly different sizes and thus we report the results separately.

\paragraph{Parameter choice.}
For each group and for each optimization method we use grid search to find the optimal method parameters.
As a criterion we use the maximum value of lower bound obtained within 25 seconds for segmentation and 200 seconds for stereo. For the bundle method we observe that larger bundles lead to better results, but increase the computation time. Value $b_s = 100$ is a compromise for both stereo and segmentation datasets. The same reasoning gives us $b_s = 10$ for LMBM and $h_r = 10$ for L-BFGS. Parameter $w_{\min}$ for (aggregated) bundle method did not affect the performance so we used $10^{-10}$ as in~\cite{Kappes12bundle}. The remaining parameters are presented in table~\ref{tbl::gridSearchParameter}.

\paragraph{Implementation details.} For L-BFGS and LMBM we use the original authors' implementations. All the other optimization routines are implemented in MATLAB. Both SMR and DD~TRW oracles are implemented in C++: for SMR we used the dynamic version of BK max-flow algorithm~\cite{Boykov04}, \cite{Kohli07dynamic}; for DD~TRW we used the domestic implementation of dynamic programming speeded up specifically for Potts model. To estimate the primal solution at the current point we use 5 iterations of ICM (Iterated Conditional Modes)~\cite{Besag86} everywhere. We normalize all the energies in such a way that the energy of the $\alpha$-expansion result corresponds to 100, the sum of minimums of unary potentials~-- to 0. All our codes are available online.\!\footnote{\url{https://github.com/aosokin/submodular-relaxation}}

\begin{table}
\renewcommand{\arraystretch}{1.3}
\caption{The optimal parameter values chosen by the grid search in sec.~\ref{sec::exp::optimization}.
\label{tbl::gridSearchParameter}}
\centering
\begin{tabular}{m{0.1cm}|m{1.35cm}||m{2.5cm}|m{2.5cm}}
\hline
& \bfseries Method & \bfseries Segmentation & \bfseries Stereo\\
\hline\hline
\multirow{4}{*}{\begin{sideways}DD~TRW\end{sideways} }
& subgradient
&  $\gamma = 0.3$
& $\gamma = 0.1$
\\ \cline{2-4}
& bundle
& $\gamma{=}0.01$, $m_L{=}0.05$, $w_{\max}{=}1000$
& $\gamma{=}0.01$, $m_L{=}0.05$, $w_{\max}{=}2000$
\\\cline{2-4}
& \parbox{1cm}{aggr. \\bundle}
& $\gamma{=}0.02$, $w_{\max}{=}100$,\!\!\! $m_r{=}0.0005$
& $\gamma{=}0.01$, $m_r{=}0.002$, $w_{\max}{=}500$
\\
\hline\hline
\multirow{4}{*}{\begin{sideways}SMR\end{sideways}}
& subgradient
& $\gamma = 0.7$
& $\gamma = 0.3$
\\ \cline{2-4}
& bundle
&  $\gamma{=}0.1$, $m_L{=}0.2$, $w_{\max}{=}100$
&  $\gamma{=}0.01$, $m_L{=}0.1$, $w_{\max}{=}100$
\\ \cline{2-4}
& \parbox{1cm}{aggr. \\bundle}
& $\gamma{=}0.02$, $m_r{=}0.001$, $w_{\max}{=}500$
& $\gamma{=}0.02$, $m_r{=}0.001$, $w_{\max}{=}1000$
\\
\hline
\end{tabular}
\end{table}

\paragraph{Results.} Fig.~\ref{fig::optimization} shows the plots of the averaged lower bounds vs. time for stereo and segmentation instances. See fig.~7 in the supplementary material for more detailed plots. We observe that SMR outperforms DD~TRW for almost all optimization methods and for the segmentation instances it outperforms TRW-S as well. When $\alpha$-expansion was not able to find the global minimum of the energy (1 segmentation instance and all stereo instances) all the tested methods obtained primal solutions with lower energies than the energies obtained by $\alpha$-expansion. 

\begin{figure}[t]
\centering
\includegraphics[width=8cm]{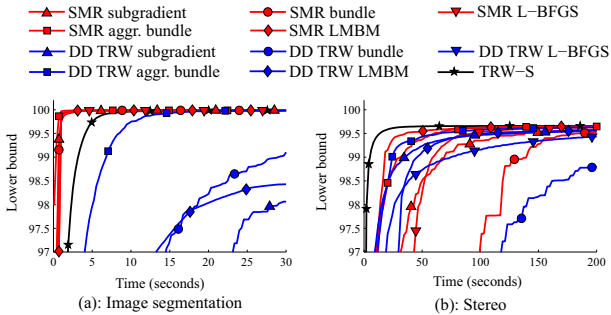}
\caption{The results of experiment~\ref{sec::exp::optimization} comparing SMR, DD~TRW (both with different optimization routines) and TRW-S on segmentation and stereo datasets.
\label{fig::optimization}}
\end{figure}

\subsection{High-order potentials}
In this section we evaluate the performance of the SMR approach on the high-order robust~$\mathcal{P}^n$-Potts model proposed by~\citet{kohli2009robust} for the task of semantic image segmentation.

\paragraph{Dataset.}
We use the two models both constructed for the 21-class MSRC dataset~\cite{Shotton06}.
The segmentation of an image is constructed via the minimization of the energy consisting of unary, pairwise and high-order potentials.
The parameters of the unary and the pairwise potentials are obtained on the standard training set via the piecewise training procedure.
The unary and pairwise potentials are based on the TextonBoost model~\cite{Shotton06} and are computed by the ALE library~\cite{Ladicky09ale}.\!\footnote{\url{http://www.inf.ethz.ch/personal/ladickyl/ALE.zip}}
The unary potentials are learned via a boosting algorithm applied to a number of local features.
The pairwise potentials are contrast-sensitive Potts based on the 8-neighborhood grid.
The high-order potentials are robust~$\mathcal{P}^n$-Potts~\eqref{eq::robustPnPotts} and are based on the superpixels (image regions) computed by the mean shift segmentation algorithm~\cite{Comaniciu02}.

The two models differ in the number of segments used to construct the high-order potentials. In the first one we use segments produced by one mean-shift segmentation with the bandwidth parameters set to~$7$ for the spatial domain and~$9.5$ for the LUV-colour domain. In the second model we use the three segmentations computed with the following pairs of parameters: $(7, 6.5)$, $(7, 9.5)$, $(7, 14.5)$~-- the first number corresponds to the spatial domain, the second one~-- to the colour domain. In both models for potential~$\theta_C$ based on a set of nodes~$C \subseteq \mathcal{V}$ parameters $Q$ and $\gamma$ are selected as recommended by~\citet{kohli2009robust}:
$Q = 0.1 |C|$, $\gamma = |C|^{0.8} (0.2 + 0.5 G(C))$ where $G(C) = \exp( - 12 \|\sum_{i \in C}(\vec{f}_i - \mu)^2 \| / |C|)$, $\mu = (\sum_{i \in C} \vec{f}_i) / |C|$ and $\vec{f}_i$ is the color vector in the RGB format: $\vec{f}_i \in [0, 1]^3$\!\!. We refer to the two models as 1-seg and 3-seg models.

\paragraph{Baseline.} As a baseline we choose a ``generic optimizer'' by~\citet{Komodakis09highorder}. We implement it as a dual decomposition method with the following decomposition of the energy: all pairwise potentials are split into vertical,  horizontal, and diagonal (2 directions) chains, each high-order potential forms a separate subproblem, unary potentials are evenly distributed between the horizontal and the vertical forests. We refer to this method as CWD.

Within the CWD method the dual function depends on $|\mathcal{V}| |\mathcal{P}| (3 + h)$ variables. Here $h$ is the average number of high-order potentials incident to individual nodes. For the 1-seg model $h$ equals~$1$ and for the 3-seg model~$h = 3$. The dual function in the SMR method  depends on $|\mathcal{V}|$ variables. In our experiments the CWD's duals have 5725440 and 8588160 variables for the 1-seg and 3-seg models respectively, the SMR's duals have 68160 variables in both models. We used the L-BFGS~\cite{Lewis08} method with step size tolerance of $10^{-2}$ and maximum number of iterations of $300$ as the optimization routine.

\paragraph{Results.}
Fig.~\ref{fig::toyProblem} reports the results of the CWD and SMR methods applied to both segmentation models on the test set of the MSRC dataset (256 images).
For each number of oracles calls we report the lower bounds and the energies.
Lower bounds are the values of the dual functions~$D(\vec{\lambda})$ obtained at the corresponding oracle calls.
Energies are computed for the labelings obtained by assigning the first possible labels to the nodes where subproblems disagree.
To aggregate the results we subtract from each energy the constant such that the best found lower bound equals zero.

We observe that SMR converges faster than CWD in terms of the value of the dual and obtains better primal solutions. Note, that the more sophisticated version of CWD, PatB~\cite{Komodakis09highorder}, could potentially improve the performance only when the high-order potentials intersect. For the 1-seg model high-order potentials do not intersect so PatB is equivalent to CWD.

\begin{figure}
\centering
\includegraphics[width=8cm]{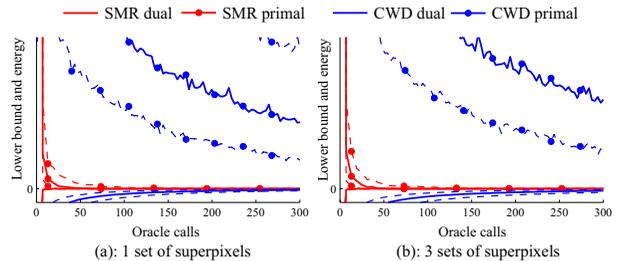}
\caption{Energies/lower bounds produced by SMR and CWD on the image segmentation models. All blue curves correspond to CWD, red curves~-- to SMR.
Plots (a) and (b) correspond to the 1-seg and 3-seg models respectively.
For each point of each plot we report the results aggregated across the standard test set: the medians (solid lines), the lower and the upper quartiles (dashed lines). This means that for each plot 50\% of all curves pass between the dashed lines.
\label{fig::toyProblem}}
\end{figure}

\paragraph{Comparison against $\alpha$-expansion.} We report the comparison of the SMR method against the original $\alpha$-expansion-based algorithm of~\citet{kohli2009robust}.
We run SMR (L-BFGS optimization) and $\alpha$-expansion methods on the energies constructed for the test set of the MSRC dataset (256 instances) as described in the previous section.
The SMR approach provides the certificate of optimality for 194 (out of 256) instances for the 1-seg model and for 197 (out of 256) instances for the 3-seg model.
The $\alpha$-expansion method finds the energy equal to the lower bound obtained by SMR in 145 and 147 cases for the two models respectively.
We exclude energies where both methods found the global minimum from further analysis.
We subtract from each energy such a constant that the best found lower bound equals zero and compare the obtained energies in table~\ref{tbl::robustPnPott}.
For each method and each model we report the median (50\%), lower (25\%) and upper (75\%) quartiles of obtained energies.

The median SMR convergence time up to accuracy of $10^{-2}$ for the 1-seg model is 23.4 seconds for the cases with no integrality gap and 136.3 seconds for the cases with the gap. For the 3-seg model the analogous numbers are 24.0 seconds and 164.3 seconds. The median running time for the $\alpha$-expansion for the two models was 2.6 and 6.3 seconds.

\begin{table}
\renewcommand{\arraystretch}{1.3}
\caption{Comparison of the energies obtained by the SMR and the~$\alpha$-expansion~\cite{kohli2009robust} algorithms on the two models with robust $\mathcal{P}^n$-Potts potentials.
\label{tbl::robustPnPott}}
\centering
\begin{tabular}{c|cc|cc}
\hline
  & \multicolumn{2}{c|}{ {\bfseries 1 segmentation} }  &  \multicolumn{2}{c}{ {\bfseries 3 segmentations}}\\
  \hline
percentile  &  $\alpha$-exp. & SMR & $\alpha$-exp. & SMR \\
\hline
25\% & 2.34  & 0     & 2.40  &  0    \\
50\% & 9.11  & 0.04  & 9.22  &  0.44 \\
75\% & 31.60 & 5.22  & 25.46 &  4.92 \\
\hline
\end{tabular}
\end{table}

\subsection{Global linear constraints}

In this section we experiment with the ability of SMR to take into account global linear constraints on the indicator variables. We compare the performance of SMR against the available baselines. See \citet{Osokin11cvpr} for the evaluation of SMR on interactive image segmentation and magnetogram segmentation tasks.

\paragraph{Dataset.}
We use a synthetic dataset similar to the one used by~\citet{Kolmogorov06trws}. We generate 10-label problems with $50 \times 50$ 4-neighborhood grid graphs. The unary potentials are generated i.i.d.:  $\theta_{ip} \sim \mathcal{N}(0, 1)$. The pairwise potentials are Potts with weights generated i.i.d. as an absolute value of $\mathcal{N}(0, 0.5)$. As global constraints we've chosen strict class size constraints~\eqref{eq::strictAreaConstr}. Class sizes are deliberately set to be significantly different from the global minimum of the initial unconstrained energy. In particular, for class $p=1,\dots,10$ we set its size to be $|\mathcal{V}| p /\! \sum_{q=1}^{\raise -.2em \hbox{\scriptsize $10$}} q$.

\paragraph{Baselines.}
First, we can include the global linear constraints in the double-loop scheme where the inner loop corresponds to the TRW-like algorithm that solves the standard LP-relaxation and the outer loop solves the Lagrangian relaxation of the global constraints:
\begin{align}
\label{eq::GTRW}
\max_{\vec{\xi}, \vec{\pi}} \;& LB_{TRW}(\tilde{\vec{\theta}}) - \!\sum_{m = 1}^M \xi_m c^m - \!\sum_{k = 1}^K \pi_k d^k, \\
\notag
\st \; &
\tilde\theta_{jp} = \theta_{jp} + \!\!\sum_{m = 1}^M \xi_mw_{jp}^m + \!\!\sum_{k= 1}^K \pi_k v_{jp}^k, \: \forall j \in \mathcal{V}, \: \forall p \in \mathcal{P},
\\
\notag &
\tilde\theta_{ij,pq} = \theta_{ij,pq}, \quad \forall \{i,j\} \in \mathcal{E},\; \forall p, q \in \mathcal{P},
\\
\notag &
\pi_k\ge 0, \quad \forall k = 1,\dots, K.
\end{align}
Here~$LB_{TRW}(\tilde{\vec{\theta}})$ is the lower bound (we use TRW-S~\cite{Kolmogorov06trws}) on the global minimum of the energy with potentials defined by~$\tilde{\vec{\theta}}$. We will refer to this algorithm as GTRW.

Another baseline is based on the MPF framework of~\citet{Woodford09mpf}. Let~$LB_{TRW}(\check{\vec{\theta}}(\vec{\mu}))$ be the solution of the standard LP-relaxation of the energy with modified unary potentials: $\check{\theta}_{jp}(\vec{\mu}) = \theta_{jp} - \mu_{jp}$, $\check{\theta}_{ij,pq}(\vec{\mu}) = \theta_{ij,pq}$. Define
\begin{align}
\label{eq::mpfTask}
\Phi_{G}(\vec{\mu}) = \min_{\vec{y}} \quad & \langle \vec{\mu}, \vec{y} \rangle, \\
 \notag \st \quad & \sum_{j \in \mathcal{V} }y_{jp}= c_p \ge 0, \quad \forall p \in \mathcal{P},
 \\
 \notag  &
 0 \leq y_{jp} \leq 1, \quad \forall j \in \mathcal{V}, \; \forall p \in \mathcal{P}.
\end{align}
Task~\eqref{eq::mpfTask} is a well-studied transportation problem (can be solved when $\sum_{p \in \mathcal{P}}c_p=|\mathcal{V}|$).

Function $LB_{TRW}(\check{\vec{\theta}}(\vec{\mu})) + \Phi_{G}(\vec{\mu})$ is concave and piecewise linear as a function of~$\vec{\mu}$. In this experiment we compute~$LB_{TRW}(\check{\vec{\theta}}(\vec{\mu}))$ with TRW-S and $\Phi_{G}(\vec{\mu})$~-- with the simplex method. We refer to this method as MPF.

\paragraph{Results.} Fig.~\ref{fig:toyAveragePlots} shows the convergence of SMR, DD~TRW, and MPF averaged over 50 randomly generated problems. For MPF we do not consider the time required to solve the transportation problem $\Phi_{G}(\vec{\mu})$ since we do not have the efficient implementation for the simplex method.
For all the methods we report the lower bounds, the energies and the total constraint violation of the obtained primal solutions.\!\footnote{
The choice of primal solution is performed by the following heuristics: in SMR all pixels with conflicting labels in the same connected component were assigned to the same randomly selected label from the set of conflicting labels; in GTRW and  MPF primal solution was chosen as the labeling generated by TRW-S.
}
Based on fig.~\ref{fig:toyAveragePlots} we conclude that all three methods converge to the same point
but the SMR lower bound converges faster and hence a primal solution with both low energy and small constraint violation can be found.

\begin{figure}[t]
\centering
\includegraphics[width=8.2cm]{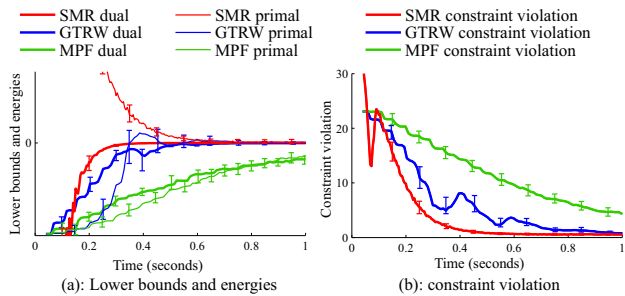}
\caption {Comparison of SMR, GTRW and MPF methods.
Plot (a) shows the lower bounds and the energies obtained by each method. Note that the energies can be lower than the lower bounds because the labelings violate the global hard constraints. Plot (b) shows the total constraint violation of the primal solution (measured in the percentage of pixels that have to be recolored to make the constraints consistent). For each curve we report the median, the lower and the upper quartiles.  \label{fig:toyAveragePlots}}
\end{figure}

\section{Conclusion \label{sec::conclusion}}
In the paper we propose a novel submodular relaxation framework for the task of MRF energy minimization. In comparison to the existing algorithms based on the decomposition of the problem's hypergraph to hyperedges and trees, our approach requires fewer dual variables and therefore converges much faster especially when the number of labels is relatively small.
State-of-art convex optimization methods together with the dynamic max-flow algorithm make the optimization of the SMR dual efficient so that SMR can be potentially applied to larger problems with sparse high-order potentials.
We study the theoretical properties of the dual solution provided by SMR under various conditions and show the equivalence of SMR lower bound to the specific LP relaxation of the initial discrete problem.

\ifCLASSOPTIONcompsoc
  \section*{Acknowledgments}
\else
  \section*{Acknowledgment}
\fi
We would like to thank Vladimir Kolmogorov and Bogdan Savchynskyy for valuable discussions and the anonymous reviewers for great comments.
This work was supported by Microsoft: Moscow State University Joint Research Center (RPD 1053945), Russian Foundation for Basic Research (projects 12-01-00938 and 12-01-33085),
Skolkovo Institute of Science and Technology (Skoltech): Joint Laboratory Agreement - No 081-R dated October 1, 2013, Appendix A 2.

\appendices

\section{Additional plots for sec.~\ref{sec::exp::optimization}}
Fig.~\ref{fig::optimizationExtra} provides additional plots for the experiment described in sec.~\ref{sec::exp::optimization} of the main paper.
Fig.~\ref{fig::optimizationExtra}a-b show zoomed versions of lower-bound plots for image segmentation and stereo instances (fig.~\ref{fig::optimization}a-b of the main paper).
Fig.~\ref{fig::optimizationExtra}c-d show values of the energy (primal) achieved at each iteration.
Fig.~\ref{fig::optimizationExtra}e-f are zoomed versions of fig.~\ref{fig::optimizationExtra}c-d.

\section{Non-submodular Lagrangian}
\subsection{The family of lower bounds\label{sec::nsmrFormulation}}
Consider the task of minimizing the pairwise energy formulated in the indicator notation:
\begin{align}
\label{eq::smd::problemPairwiseIndicator}
\min_{\vec{y}}&\quad \sum_{i \in \mathcal{V}} \sum_{p \in \mathcal{P}} \theta_{ip} y_{ip} + \sum_{ \{i, j\} \in \mathcal{E}} \sum_{p,q \in \mathcal{P}} \theta_{ij,pq} y_{ip} y_{jq},\\
\label{eq::smd::discretizationContraints}
\st
&\quad y_{ip} \in \{0, 1\}, \quad \forall i \in \mathcal{V}, p \in \mathcal{P},\\
\label{eq::smd::consistencyConstraints}
 &\quad \sum_{p \in \mathcal{P}} y_{ip} = 1, \quad \forall i \in \mathcal{V},
\end{align}

Sec.~4.1 
of the main paper considers the case when all values $\theta_{ij,pq}$ of the potentials are non-positive.
In this section we elaborate on the case when it is not true.
To apply the SMR method (see sec.~4.2) in this case we have to subtract from each pairwise potential a constant equal to the maximum of its values: $\max_{p,q} \theta_{ij,pq}$. This procedure can badly affect the sparsity of the initial potentials, e.g. if the potential specifically prohibits certain configurations. If we drop the subtraction step we retain sparsity but expression~\eqref{eq::smd::problemPairwiseIndicator} is no longer guaranteed to be submodular, thus minimizing it over the Boolean cube is hard. Instead of solving this minimization exactly we can solve its standard LP relaxation using QPBO method \cite{Boros02}, \cite{Kolmogorov07qpbo} and further perform Lagrangian relaxation of the consistency constraint~\eqref{eq::smd::consistencyConstraints}. We refer to this approach as \emph{nonsubmodular relaxation} (NSMR) and describe it more formally.

\begin{figure}[t]
\centering
\includegraphics[width=8cm]{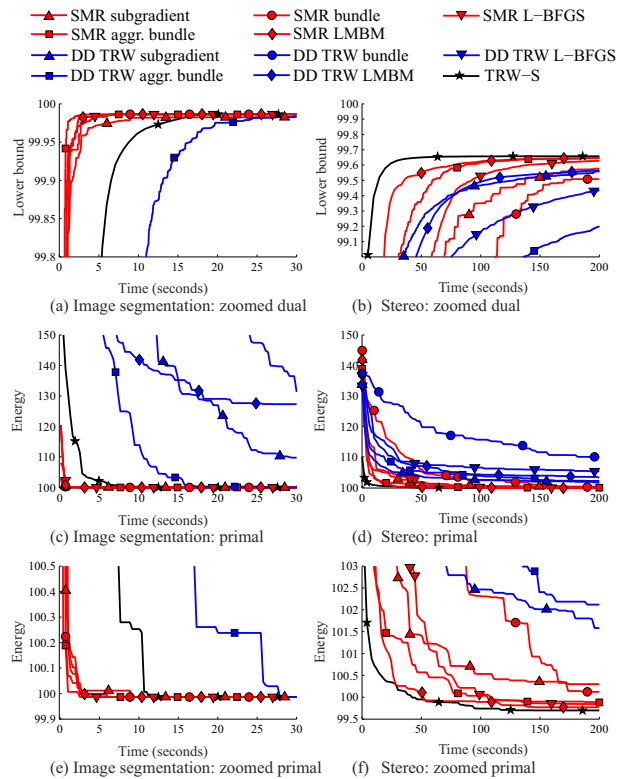}
\caption{Extra plots for the experiment in sec.~7.1.
\label{fig::optimizationExtra}}
\end{figure}

Consider the following Lagrangian:
\begin{align}
\notag
L(\vec{y}, \vec{\lambda})
&=
\sum_{i \in \mathcal{V}} \sum_{i \in \mathcal{P}} \theta_{ip} y_{ip} +
\sum_{ \{i, j\} \in \mathcal{E}} \sum_{ p, q \in \mathcal{P}} \!\!\!\!\theta_{ij,pq} y_{ip} y_{jq} \\
\label{eq::gpld::lagrangian}
&+
\sum_{i \in \mathcal{V}}\lambda_i \biggl( \sum_{p \in \mathcal{P}} y_{ip} - 1 \biggr)
\end{align}
It is a pairwise pseudo-Boolean function w.r.t. variables~$\vec{y}$ so we can apply the QPBO algorithm.
The algorithm provides us with a lower bound $D_{QPBO}(\vec{\lambda})$ on the solution of~\eqref{eq::smd::problemPairwiseIndicator}-\eqref{eq::smd::consistencyConstraints} and a \emph{partial labeling} of $\vec{y}$, i.e. each $y_i$ is assigned a value from set $\{ 0, 1, \emptyset \}$.
The QPBO lower bound is known to be equal to the bound provided by the standard LP relaxation~\cite[th.~6 and~7]{Boros02} and thus function $D_{QPBO}(\vec{\lambda})$ is concave and piecewise-linear w.r.t. $\vec{\lambda}$.

To compute a subgradient of $D_{QPBO}(\vec{\lambda})$ we use the fact that the standard LP relaxation of a pairwise energy minimization problem is half-integral~\cite[prop.~20]{Boros02}. Specifically, if QPBO assigns~$0$, $1$, or $\emptyset$  to $y_{ip}$ the corresponding unary variable of the LP-relaxation ($y_{ip}$) equals $0$, $1$, or $0.5$ respectively. Pairwise variables $y_{ij,pq}$ can be set by the following rule (see lemma~\ref{theor::gpldQpboLpLemma}):
\begin{equation}
\label{eq::gpld::defZ}
 y_{ij,pq} =
    \begin{cases}
        \min(y_{ip},y_{jq}), & \theta_{ij,pq} \leq 0, \\
        \max(0,y_{ip}+y_{jq}-1), & \mbox{otherwise}.
    \end{cases}
\end{equation}
We explore tightness of the NSMR lower bound in section~\ref{sec::nsmrTightness} and compare it experimentally against the SMR with the subtraction trick in sec.~\ref{sec::nsmrComparison}.

\subsection{Tightness of the NSMR bound\label{sec::nsmrTightness}}
\begin{theoremFixed}{8}
\label{theor::nsmrTightness}
The lower bound given by NSMR (sec.~\ref{sec::nsmrFormulation}) method applied to a pairwise energy equals to the minimum value of the following linear program:
\begin{align}
\label{eq::nsmrTheor::target}
\min_{\vec{y}}\quad
&
\sum_{i\in \mathcal{V}} \sum_{p\in \mathcal{P}}  \theta_{ip} y_{ip} +
\sum_{\{i,j\} \in \mathcal{E}} \sum_{p, q \in \mathcal{P}}  \theta_{ij,pq} y_{ij, pq},
\\
\label{eq::nsmrTheor::segmentIntConstr}
\mbox{s.t.}\;
&
y_{ip}, y_{ij,pq} \geq [0, 1],\\
\label{eq::nsmrTheor::constr1}
&
y_{ij, pq} \!\leq\! y_{i p},\:\!y_{ij, pq} \!\leq\! y_{j q},\:\!\forall \{i,j\}\!\!\in\!\mathcal{E},\:\!\forall p,q\!\in\!\mathcal{P},\\
\label{eq::nsmrTheor::constr2}
&
y_{ij, p} \!\geq\! y_{ip} \!+\! y_{jp} \!-\! 1, \: \forall \{i,j\}\!\!\in\!\mathcal{E},\:\!\forall p,q\!\in\!\mathcal{P},
\\
\label{eq::nsmrTheor::constrConsist}
&
\sum_{p\in \mathcal{P}} y_{ip} = 1, \quad \forall i \in \mathcal{V}.
\end{align}
\end{theoremFixed}
\begin{proof}
We start the proof with the lemma that formulates the standard LP relaxation of a pairwise energy in a convenient form:
\begin{lemmaFixed}{3}
\label{theor::gpldQpboLpLemma}
Consider the problem of minimizing energy
$
E(\vec{y}) = \sum_{i \in \mathcal{A}} a_i y_i + \sum_{ \{i, j\} \in \mathcal{B}} b_{ij} y_i y_j
$
w.r.t. binary variables $\vec{y} = \{ y_i \}_{i \in \mathcal{A}}$ where $\mathcal{A}$ and $\mathcal{B}$ are sets of unary and pairwise terms correspondingly.
The standard LP relaxation of the problem can be formulated as follows:
\begin{align}
\label{eq::gpld::subbproblemTarget}
\min_{y_{i}, y_{ij}}\quad
&
\sum_{i\in \mathcal{A}} a_i y_{i} + \sum_{\{i,j\} \in \mathcal{B}} b_{ij} y_{ij},
\\
\label{eq::gpld::subproblemIntConstr}
\mbox{s.t.}\quad
& y_{i}, y_{ij} \geq 0,
\\
\label{eq::gpld::subproblemConstr1}
&
y_{ij} \leq y_{i},\; y_{ij} \leq y_{j}, \quad \forall \{i,j\} \in \mathcal{B},
\\
\label{eq::gpld::subproblemConstr2}
&
y_{ij} \geq y_{i} + y_{j} - 1, \quad \forall \{i,j\} \in \mathcal{B}.
\end{align}
\end{lemmaFixed}
\begin{proof}
At first, recall how the standard LP relaxation looks like in this case:
\begin{align}
\label{eq::gpldLPSubproblemTarget}
\min_{\vec{z}}\quad
&
\sum_{i\in \mathcal{A}} a_i z_{i1} + \sum_{\{i,j\} \in \mathcal{B}} b_{ij} z_{ij, 11},
\\
\label{eq::gpldLPSubproblemNonNegConstr}
\mbox{s.t.}\quad
& z_{i1}, z_{i0}, z_{ij,11},z_{ij,01},z_{ij,10},z_{ij,00} \geq 0,
\\
\label{eq::gpldLPSubproblemConstr1}
&
z_{i1}+z_{i0}=1,
\\
\notag
&
z_{ij,10}+z_{ij,11}=z_{i1}, \quad
z_{ij,01}+z_{ij,00}=z_{i0},
\\
\label{eq::gpldLPSubproblemConstr2}
&
z_{ij,01}+z_{ij,11}=z_{j1},\quad
z_{ij,00}+z_{ij,10}=z_{j0}.
\end{align}

Given the solution of LP~\eqref{eq::gpldLPSubproblemTarget}-\eqref{eq::gpldLPSubproblemConstr2} we can construct the solution of LP \eqref{eq::gpld::subbproblemTarget}-\eqref{eq::gpld::subproblemConstr2} with the same value of the target function by setting $y_{i} = z_{i1}$, $y_{ij} = z_{ij,11}$. The feasibility of this assignment is implied by non-negativity of $\vec{z}$ and constraints~\eqref{eq::gpldLPSubproblemConstr2}.

On the other hand given the solution of~\eqref{eq::gpld::subbproblemTarget}-\eqref{eq::gpld::subproblemConstr2} we can set $z_{j1} = y_{j}$, $z_{j0} = 1 - z_{j1}$, $z_{ij,11} = y_{ij}$, $z_{ij,01}=z_{j1}-z_{ij,11}$, $z_{ij,10}=y_{i1}-z_{ij,11}$, $z_{ij,00}=z_{ij,11}+1-y_{i1}-y_{j1}$, which implies the feasibility by construction.
\end{proof}

Now we complete the proof of theorem~\ref{theor::nsmrTightness}.
QPBO lower bound $D_{QPBO}(\vec{\lambda})$ is known to be equal to the minimum of the standard LP relaxation~\cite[th.~6 and~7]{Boros02} and therefore can be expressed in form~\eqref{eq::gpld::subbproblemTarget}-\eqref{eq::gpld::subproblemConstr2} that corresponds to LP \eqref{eq::nsmrTheor::target}-\eqref{eq::nsmrTheor::constr2}. Applying Lagrangian relaxation to consistency constraints~\eqref{eq::nsmrTheor::constrConsist} (lemma~2) adds them to the LP.
\end{proof}

\section{Comparison of theorem~\ref{theor::nsmrTightness} and corollary~1 \label{sec::nsmrComparison}}
SMR can be applied to problem~\eqref{eq::smd::problemPairwiseIndicator}-\eqref{eq::smd::consistencyConstraints} in case when some $\theta_{ij,pq}\geq 0$ using the subtraction trick described in sec. 4.2. In this section we compare the LP obtained this way with the one provided by theorem~\ref{theor::nsmrTightness}. The LP given by corollary~1 together with the the subtraction trick looks as follows:\\[-0.3cm]
\begin{align}
\label{eq::LP1:pairwise}
\min_{\vec{y}}\; &\sum_{i \in \mathcal{V}}\!\sum_{p \in \mathcal{P}} \theta_{ip} y_{ip} \!+\!\!\!\!\!\!\sum_{ \{i, j\} \in \mathcal{E}} \!\!\biggl(\sum_{p,q \in \mathcal{P}} \!\!(\theta_{ij, pq} - \theta^*_{ij}) y_{ij, pq}\!+\!\theta^*_{ij}\!\biggr),\\
\notag
\st\; &y_{ip}\in[0,1],  \quad \forall i\in\mathcal{V}, \; \forall p\in\mathcal{P},\\
\notag
& y_{ij, pq} \in[0,1],  \quad \forall \{i,j\} \in \mathcal{E}, \; \forall p,q \in \mathcal{P}, \\
\label{eq::LP2cons:pairwise} & y_{ij, pq} \!\leq\! y_{i p}, \: y_{ij, pq} \!\leq\! y_{j q},
 \: \forall \{i,j\}\!\in\!\mathcal{E}, \forall p,q\!\in\!\mathcal{P},\\
\label{eq::LP3cons:pairwise}
&\sum\nolimits_{p \in \mathcal{P} } y_{ip}=1, \quad \forall i\in\mathcal{V}
\end{align}
where~$\theta^*_{ij} = \max_{p,q} \theta_{ij, pq}$.
The optimal point of LP~\eqref{eq::LP1:pairwise}-\eqref{eq::LP3cons:pairwise} is feasible for LP~\eqref{eq::nsmrTheor::target}-\eqref{eq::nsmrTheor::constrConsist} and vice versa. Note, that at the same time the polyhedra defined by the two sets of constraints are not equivalent.

We provide the experimental comparison of the integrality gap in problem~\eqref{eq::nsmrTheor::target}-\eqref{eq::nsmrTheor::constrConsist} against the integrality gap of~\eqref{eq::LP1:pairwise}-\eqref{eq::LP3cons:pairwise} and the integrality gap given by the TRW-S~\cite{Kolmogorov06trws} and DD~TRW~\cite{Komodakis10ddtrw} algorithms.

We generate an artificial data set similar to the one used in sec.~7.3 of the main paper.
Specifically we generate 50 energies of size $20 \times 20$ with 5 labels each. The unary potentials are generated i.i.d.:  $\theta_{ip} \sim \mathcal{N}(0, 1)$. The pairwise potentials are Potts with weights generated i.i.d. from $\mathcal{N}(0, 0.5)$. Note, that many of these potentials are not associative.

We report the results in table~\ref{tbl:nsmr}. We can conclude that NSMR provides much tighter relaxation than the SMR with the subtraction trick and is comparable to the relaxation provided by TRW-S.

\begin{table}
\caption{Comparison of integrality gaps provided by NSMR, SMR with the subtraction trick and
TRW-S. We report the results aggregated over 50 energies.
For the three different percentiles (the median and the two quartiles) we report the integrality gaps.
 \label{tbl:nsmr}}
\begin{center}
\begin{tabular}{c|cccc}
\hline
percentile
&
NSMR
&
Subtraction
&
TRW-S
&
DD TRW \\
\hline
25\% & 0 & 332 & 0 & 0 \\
50\% & 0.0012 & 355 & 0.0034 & 0.0021 \\
75\% & 0.1621 & 381 & 0.3188 & 0.1155 \\
\hline
\end{tabular}
\end{center}
\end{table}

\section{Proof of theorem~1}
\begin{theoremFixed}{1}\label{theor::marginAveraging}
Rule~$\lambda_j^{new} = \lambda_j^{old} + \Delta_j$, $\lambda_i^{new} = \lambda_i^{old}$, $i\neq j$, $ \delta^j_{(2)} \leq \Delta_j \leq \delta^j_{(1)}$,
when applied to a pairwise energy with associative pairwise potentials delivers the maximum of function $D(\vec{\lambda})$ w.r.t. variable~$\lambda_j$.
\end{theoremFixed}
\begin{proof}
We start with proving lemma~\ref{lemma::minMarginals}.
\begin{lemmaFixed}{4}\label{lemma::minMarginals}
Consider a pseudo-Boolean function~$F: \mathbb{B}^n \to \mathbb{R}$. Denote the difference between the min-marginals of variable~$i$ and labels~$0$ and~$1$ with~$\delta_i$: $\delta_i = MM_{i,0} F - MM_{i,1} F$. The minimum of the energy with one unary potential added can be computed explicitly:
\begin{equation*}
\min_{\vec{z} \in \mathbb{B}^n } \left( F(\vec{z}) + \delta z_i \right) =
\begin{cases}
MM_{i,0} F, \quad &\delta \geq \delta_i, \\
MM_{i,1} F + \delta, \quad & \delta \leq \delta_i.
\end{cases}
\end{equation*}
\end{lemmaFixed}
\begin{proof}
Recall the definition of the min-marginals of~$F$: $MM_{i,0} F = \min_{\vec{z}:\; z_i = 0} F(\vec{z})$ and~$MM_{i,1} F = \min_{\vec{z}:\; z_i = 1} F(\vec{z})$. Min-marginals of the extended function~$F\vec{z} + \delta z_i$ of variables~$z_i$ can be computed analitically:
$
MM_{i,0} (F + \delta z_i) = \min_{\vec{z}:\; z_i = 0} F(\vec{z})
$
and
$
MM_{i,1} (F + \delta z_i) = \min_{\vec{z}:\; z_i = 1} F(\vec{z}) + \delta.
$
Equality
$$
\min_{\vec{z} \in \mathbb{B}^n } F\vec{z} + \delta z_i = \min( MM_{i,0} (F + \delta z_i), \; MM_{i,1} (F + \delta z_i) )
$$
finishes the proof.
\end{proof}
Let us finish the proof of theorem~\ref{theor::marginAveraging}.
We exploit the fact that in case of pairwise associative potentials the dual function can be represented as the following sum:
\begin{equation}
\label{eq::dualAppendix}
D(\vec{\lambda}) = \sum_{p \in \mathcal{P}} \min_{\vec{y}_p}\Phi_p( \vec{y}_p, \vec{\lambda}) - \sum_{i \in \mathcal{V}} \lambda_i
\end{equation}
where each summand~$\Phi_p( \vec{y}_p, \vec{\lambda})$ is a submodular function w.r.t. variables~$\vec{y}_p$.
It suffices to show that $D(\vec{\lambda}^{new}) \geq D(\vec{\lambda}^\delta)$ where  $\lambda^\delta_i = \lambda^{old}_i$, $i\neq j$, and $\lambda^\delta_j = \lambda^{old}_j + \delta$, $\delta \in \mathbb{R}$.

Consider case when~$\delta < \Delta_j$. We have $\delta < \delta^j_{(1)}$ where
$$
\delta^j_{(1)}
\!\!=\!\!
MM_{j,0} \; \Phi_{p^{(1)}_j}(\vec{y}_{p^{(1)}_j},\vec{\lambda}^{old}) - MM_{j,1}\; \Phi_{p^{(1)}_j}(\vec{y}_{p^{(1)}_j},\vec{\lambda}^{old}).
$$
Lemma~\ref{lemma::minMarginals} and inequality~$\delta < \delta^j_{(1)}$ ensure that
\begin{equation}
\label{eq::proofMM1}
\min_{\vec{y}_{p^{(1)}_j}}\Phi_{p^{(1)}_j}(\vec{y}_{p^{(1)}_j},\vec{\lambda}^{new})
=
\min_{\vec{y}_{p^{(1)}_j}} \Phi_{p^{(1)}_j}(\vec{y}_{p^{(1)}_j},\vec{\lambda}^{\delta})+\Delta_j - \delta.
\end{equation}
For all other summands of the dual~\eqref{eq::dualAppendix} inequalities
\begin{equation}
\label{eq::proofMM2}
\min_{\vec{y}_p}\Phi_p(\vec{y}_p,\vec{\lambda}^{new}) \ge \min_{\vec{y}_p}\Phi_p(\vec{y}_p,\vec{\lambda}^{\delta}), \quad \forall p \ne p^{(1)}_j
\end{equation}
hold because $\Delta_j > \delta$.

After summing up~\eqref{eq::proofMM1} and~\eqref{eq::proofMM2} we use~\eqref{eq::dualAppendix} and get
$
D(\vec{\lambda}^{new})
\geq
D(\vec{\lambda}^{\delta}).
$

Now consider case~$\delta > \Delta_j$. We have $\delta > \delta^j_{(2)}$ where
$$
\delta^j_{(2)}
\!\!=\!\!
MM_{j,0} \; \Phi_{p^{(2)}_j}(\vec{y}_{p^{(2)}_j},\vec{\lambda}^{old}) - MM_{j,1}\; \Phi_{p^{(2)}_j}(\vec{y}_{p^{(2)}_j},\vec{\lambda}^{old}).
$$
Lemma~\ref{lemma::minMarginals} and inequality~$\delta > \delta^j_{(2)}$ ensure that
\begin{equation}
\label{eq::proofMM5}
\min_{\vec{y}_p}\Phi_p(\vec{y}_p,\vec{\lambda}^{new})
=
\min_{\vec{y}_p}\Phi_p(\vec{y}_p,\vec{\lambda}^{\delta}), \quad \forall p \ne p^{(1)}_j,
\end{equation}
and
\begin{equation}
\label{eq::proofMM6}
\min_{\vec{y}_{p^{(1)}_j}}\Phi_{p^{(1)}_j}(\vec{y}_{p^{(1)}_j},\vec{\lambda}^{new})
\geq
\min_{\vec{y}_{p^{(1)}_j}} \Phi_{p^{(1)}_j}(\vec{y}_{p^{(1)}_j},\vec{\lambda}^{\delta})+\Delta_j - \delta.
\end{equation}
After summing up~\eqref{eq::proofMM5} and~\eqref{eq::proofMM6} we use~\eqref{eq::dualAppendix} and get
$
D(\vec{\lambda}^{new})
\geq
D(\vec{\lambda}^{\delta}).
$
\end{proof}

\section{Proof of theorem~2}
\begin{theoremFixed}{2}
\label{theor::wac}
A maximum point~$\vec{\lambda}^*$ of the dual function~$D(\vec{\lambda})$ satisfies the weak agreement condition.
\end{theoremFixed}
\begin{proof}
Let us assume the contrary: point~$\vec{\lambda}^*$ is not a weak agreement point. This implies that for some node $j \in \mathcal{V}$ either there exist labels $p \neq q$ such that $Z_{jp}(\vec{\lambda}^*)=Z_{jq}(\vec{\lambda}^*)=\{1\}$ or $Z_{jp}(\vec{\lambda}^*)=\{0\}$ holds for all labels~$p$. In the first case we can take $\vec{\lambda}^1$ such that for all $i\ne j$ holds $\lambda_i^1=\lambda_i^*$ and $\lambda_j^1=\lambda_j^*+\varepsilon$ where $\varepsilon>0$. If $\varepsilon$ is small enough $\{1\}$ still belongs to both $Z_{jp}(\vec{\lambda}^1)$ and $Z_{jq}(\vec{\lambda}^1)$. Hence $\vec{y}^*\in \Argmin L(\vec{y},\vec{\lambda}^1)$ but $D(\vec{\lambda}^1) > D(\vec{\lambda}^*)$ contradicting the assumption that $\vec{\lambda}^*$ is a maximum point point of $D(\vec{\lambda})$. The second case can be proven by considering $\lambda_j^1=\lambda_j^*-\varepsilon$.
\end{proof}

\section{Proof of theorem~3}
\begin{theoremFixed}{3}
\label{theor::minMarginalWeakAgreement}
Coordinate ascent algorithm (fig.~2) when applied to a pairwise energy with associative pairwise potentials returns a point that satisfies the weak agreement condition.
\end{theoremFixed}
\begin{proof}
When converged the algorithm returns a point $\vec{\lambda}$ where~$\delta^j_p \leq 0 \leq \delta^j_{(1)}$, $p \in \mathcal{P} \setminus \{p_j^{(1)}\}$, $j \in \mathcal{V}$. Consequently $1 \in Z_{j, p_j^{(1)}}$ and $0 \in Z_{j, p}$, $p \in \mathcal{P} \setminus \{p_j^{(1)}\}$, which means that the weak agreement is satisfied.
\end{proof}

\section{Proof of lemma 1}
\begin{lemmaFixed}{1}
\label{theor::lpHighOrderLemma}
Let~$y_1, \dots, y_K$ be real numbers belonging to the segment~$[0, 1]$, $K > 1$.
Consider the following linear program:
\begin{align}
\label{eq::lemmaHighOrderLp:target}
\min_{z, z_1, \dots, z_K \in [0, 1]} \quad&  z(K - 1) - \sum_{k = 1}^K z_k \\
\label{eq::lemmaHighOrderLp:constraint}
\st \qquad\;\; &
z_k \leq z, \; z_k \leq y_k.
\end{align}
Point~$z = z_1 = \dots = z_K = \min_{k = 1, \dots, K} y_k$ delivers the minimum of LP~\eqref{eq::lemmaHighOrderLp:target}-\eqref{eq::lemmaHighOrderLp:constraint}.
\end{lemmaFixed}
\begin{proof}
Without restricting the generality let us assume that $y_1 \leq \dots \leq y_K$.
Denote the solution of ~\eqref{eq::lemmaHighOrderLp:target}-\eqref{eq::lemmaHighOrderLp:constraint}, with~$z^*, z^*_1, \dots, z^*_K$, and the value of solution~-- with~$f^*$.
Variable~$z$ has a positive coefficient in the linear function~\eqref{eq::lemmaHighOrderLp:target} so $z^* = \max_{k = 1, \dots, K} z^*_k$. Analogously $z^*_k = \min(y_k, z^*)$.

Let us show that $z^* \geq y_1$. Assume the opposite, i.e.~$\varepsilon =  y_1 - z^* > 0$.  Point~$z^* + \varepsilon$, $z_1^* + \varepsilon, \dots, z^*_K + \varepsilon$ is feasible and at it function~\eqref{eq::lemmaHighOrderLp:target} equals $f^* - \varepsilon$, which contradicts the optimality of~$f^*$.

Let us show that $z^* \leq y_2$. Assume the opposite. There are two case possible:
\begin{enumerate}
\item $y_K < z^*$. Denote $\varepsilon =  z^* - y_K > 0$. Point~$z^* - \varepsilon$, $z_1^*, \dots, z^*_K$ is feasible and delivers the value of~\eqref{eq::lemmaHighOrderLp:target} equal to $f^* - \varepsilon$, which contradicts the optimality of~$f^*$.
\item $y_\ell  < z^* \leq y_{\ell + 1}$, for some~$\ell \in \{2, \dots, K\}$. Denote $\varepsilon =  z^* - y_\ell > 0$. Point~$z^* - \varepsilon$, $z_1^*, \dots, z_\ell^*, z_{\ell + 1}^* - \varepsilon, \dots, z^*_K - \varepsilon$ is feasible and delivers the value of~\eqref{eq::lemmaHighOrderLp:target} equal to $f^* - \varepsilon(\ell - 1)$, which contradicts the optimality of~$f^*$.
\end{enumerate}

All points of form $z^* \in [y_1, y_2]$, $z_1^* = y_1$, $z^*_2 = \dots = z^*_K = z^* $ are feasible and deliver the optimal value~$f^*$ to function~\eqref{eq::lemmaHighOrderLp:target}. The point specified in the statement of the lemma is of this form.
\end{proof}




\bibliographystyle{IEEEtranSN}
\bibliography{library}

\vspace{-1cm}
\begin{biography}[{\includegraphics[width=1in,height=1.25in,clip,keepaspectratio]{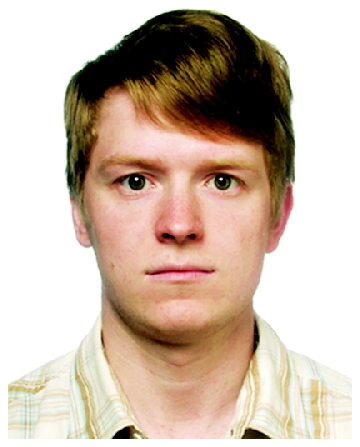}}]{Anton Osokin}
received his M.Sc. (2010) and Ph.D. (2014) degrees from Moscow State University where he worked in Bayesian Methods research group under supervision of Dmtry Vetrov. He is currently a postdoctoral researcher with SIERRA team of INRIA and \'{E}cole Normale Sup\'{e}rieure. His research interests include machine learning, computer vision, graphical models, combinatorial optimization. In 2012-2014 he received the president scholarships for young researches.
\end{biography}


\vspace{-1cm}
\begin{biography}[{\includegraphics[width=1in,height=1.25in,clip,keepaspectratio]{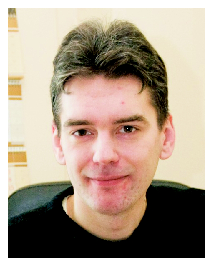}}]{Dmitry P. Vetrov}
received the M.Sc. and Ph.D. degrees from Moscow State University in 2003 and 2006, respectively. Currently he is the head of Department at the faculty of Computer Science in Moscow Higher School of Economics. He is leading Bayesian Methods research group in MSU and HSE. His research interests include deep learning, computer vision, Bayesian inference and graphical models. In 2010-2014 he received the president scholarships for young researchers.
\end{biography}






\end{document}